\documentclass[twoside]{article}
\usepackage[preprint]{aistats2027}

\usepackage[T1]{fontenc}
\usepackage{lmodern}
\usepackage{microtype,graphicx,booktabs,longtable,array,xurl}
\usepackage[round]{natbib}

\usepackage{amsmath,amssymb,amsthm}
\usepackage{xcolor,hyperref}
\usepackage[capitalize,noabbrev]{cleveref}
\theoremstyle{plain}
\newtheorem{theorem}{Theorem}
\newtheorem{proposition}[theorem]{Proposition}
\newtheorem{corollary}[theorem]{Corollary}
\newtheorem{lemma}[theorem]{Lemma}
\theoremstyle{definition}

\newcommand{\R}{\mathbb R}
\newcommand{\E}{\mathbb E}
\newcommand{\Prob}{\mathbb P}
\newcommand{\Id}{I_d}
\newcommand{\sphere}{\mathbb S^{d-1}}
\newcommand{\ind}{\mathbf 1}
\newcommand{\risk}{\mathcal R}
\newcommand{\Pois}{\operatorname{Pois}}
\newcommand{\SF}{\overline F}
\DeclareMathOperator{\tr}{tr}

\DeclareMathOperator{\Cov}{Cov}

\DeclareMathOperator{\row}{row}
\DeclareMathOperator{\Beta}{Beta}
\DeclareRobustCommand{\lref}[1]{\lowercase{\ref{#1}}}

\newcommand{\parhead}[1]{\par\vskip-.5\parskip\noindent\textbf{#1}\enspace\ignorespaces}
\makeatletter
\def\thm@space@setup{\thm@preskip=2pt plus 1pt minus 1pt \thm@postskip=\thm@preskip}
\renewenvironment{proof}[1][\proofname]{\par
  \pushQED{\qed}%
  \normalfont \topsep2\p@\@plus2\p@\relax
  \trivlist
  \item[\hskip\labelsep\itshape #1\@addpunct{.}]\ignorespaces
}{\popQED\endtrivlist\@endpefalse}
\makeatother
\newenvironment{compactitems}{%
  \begin{list}{\textbullet}{%
    \setlength{\leftmargin}{1.2em}%
    \setlength{\labelwidth}{0.7em}%
    \setlength{\labelsep}{0.5em}%
    \setlength{\topsep}{0.2ex}%
    \setlength{\partopsep}{0pt}%
    \setlength{\itemsep}{0.2ex}%
    \setlength{\parsep}{0pt}%
  }%
}{\end{list}}
\newcommand{\coderepo}{\url{https://github.com/sneddy/shift-repair-frontier}}

\hypersetup{colorlinks=true,linkcolor=black,citecolor=black,urlcolor=blue!55!black,
pdfauthor={Anuar Aimoldin, Yankai Chen, Ayana Mussabayeva, Xue Liu},pdftitle={Detecting a Shift Is Not Enough: Exact Minimax Limits of Linear Representation Repair}}
\begin{document}
\runningtitle{Exact Minimax Limits of Linear Representation Repair}
\runningauthor{Aimoldin, Chen, Mussabayeva, Liu}
\raggedbottom
\twocolumn[
\aistatstitle{Detecting a Shift Is Not Enough:\\Exact Minimax Limits of Linear Representation Repair}
\aistatsauthor{Anuar Aimoldin \And Yankai Chen \And Ayana Mussabayeva \And Xue Liu}
\aistatsaddress{Mohamed bin Zayed University of Artificial Intelligence (MBZUAI)\\ \texttt{anuar.aimoldin@mbzuai.ac.ae}} ]

\begin{abstract}
A mean shift between two data sources can be easy to detect but hard to remove without substantially changing their representations.
We cast its removal as a statistical decision problem: from noisy differences between paired calibration measurements in $\R^d$, learn one linear map, applied to both sources under a hard distortion budget, that leaves as little of the shift as possible on fresh data.
We derive the exact finite-sample minimax risk over all such maps, $(d-k)\,\E[1/(d+2J)]$ with $J\sim\Pois(\kappa/2)$, where the budget allows deleting $k$ directions and $\kappa$ is the calibration signal-to-noise ratio.
Projecting out the mean calibration difference attains it without knowing $\kappa$ or the noise scale.
This exposes a detection--repair gap: detecting the shift needs only $\kappa\gg\sqrt d$, whereas removing a fixed fraction of it at constant distortion needs $\kappa\asymp d$, as for estimating its direction.
Standard linear concept erasers (MP, SAL, LEACE) remove the same calibration difference, so the formula gives, before fitting, exactly how much shift they leave on fresh data and how much calibration a target requires.
The limit is robust: pairing keeps it exact for non-Gaussian shared content, the projection keeps its guarantee under anisotropic noise, and selective abstention cannot close the gap.
On paired clinical and wearable sleep EEG, where differences between participants act as calibration noise, the formula predicts the device shift left in new participants, and more recordings per person soon stop helping.
Together, these results tell whether a correction that falls short needs a better method, more recordings, or more participants.
\end{abstract}

\section{Introduction}
\label{sec:intro}
If the direction of a mean shift between two data sources were known, removing it would be simple: an orthogonal projection deletes it exactly.
In practice the direction must be estimated from finite, noisy data, and it stays uncertain even when the shift itself is easy to detect.
Projecting out an inaccurately estimated direction removes ordinary variation while preserving much of the shift, so a correction that falls short may reflect an inadequate rule or simply too little information to learn a better one.

\begin{figure*}[t]
 \centering
 \includegraphics[width=\linewidth]{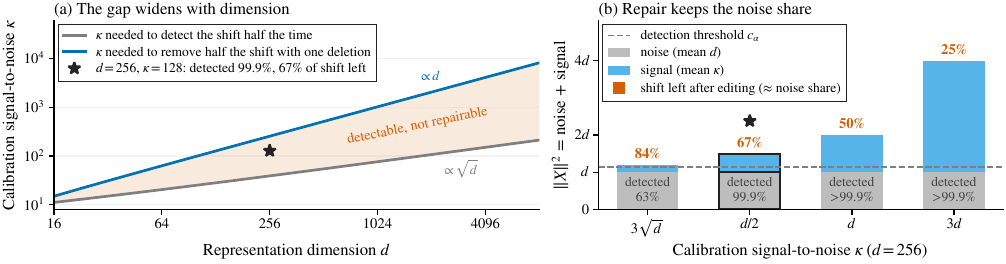}
 \caption{\textbf{Detecting a shift is not enough to repair it.}
 (a) Detection needs $\kappa\propto\sqrt d$ and repair $\kappa\propto d$, so the wedge of shifts that are detectable but not repairable widens with dimension. Detection is a test of ``no shift'' at the $5\%$ significance level, and repair uses the best linear map that deletes one direction.
 (b) At $d=256$, the squared length $\|X\|^2$ of the normalized mean calibration difference has expected value $d+\kappa$: noise $d$ (grey) plus signal $\kappa$ (blue). The test needs only the total to clear the dashed threshold, whereas editing leaves about the grey share of the bar (orange).
 The star marks the same setting in both panels. All values are exact.}
 \label{fig:concept}
\end{figure*}

Paired measurements make this concrete: the same session recorded by two EEG devices, the same participant scanned at two sites \citep{yamashita2019harmonization}, or the same input with a spurious attribute changed \citep{bai2025ncm}.
Each pair shares its content, so the difference within a pair isolates the shift.
From such pairs we learn an editor: one linear map shared by both sources and held to a hard distortion budget, and on new data we measure the share of the source separation that remains.
Without pairs, the same problem is linear concept erasure.
\emph{This raises three questions: how much of the shift can any such correction remove at best, what do standard erasers achieve, and what kind of calibration helps?}

\parhead{How much any linear correction can remove.}
\Cref{thm:frontier} answers the first question exactly.
Over all shared linear editors whose budget $b$ allows deleting $k\ge1$ directions, the minimax risk, which we call the frontier, is $\risk_d^*(b,\kappa)=(d-k)\,\E[1/(d+2J)]$ with $J\sim\Pois(\kappa/2)$ and $\kappa$ the calibration signal-to-noise ratio.
The simplest editor attains it: project out the mean paired difference.
The estimated direction keeps the noise share $d/(d+\kappa)$ of the calibration energy, which \cref{eq:bounds} shows is the optimal residual up to $1/(d+\kappa)$, whereas a test only needs the signal to rise above the $O(\sqrt d)$ noise fluctuation (\cref{fig:concept}b).
Detection therefore succeeds long before repair: at $d=256$ and $\kappa=128$, a level-$0.05$ test detects the shift with probability $0.999$, yet the best editor allowed to delete one direction leaves about two thirds of the squared separation (\cref{fig:concept}a).
These two scales are familiar from detecting and estimating a spike \citep{johnstone2009consistency,berthet2013optimal,cai2013sparse}.
What is new is the reduction that makes the limit bind every shared linear editor under the budget rather than one estimator, and its exact finite-sample value.

\parhead{What existing erasers pay.}
Mean projection (MP; \citealp{haghighatkhah2022mp}), centered rank-one SAL \citep{shao2023sal} and binary LEACE \citep{belrose2023leace} all delete the empirical mean difference (\cref{cor:empirical-erasers}), so the frontier is their exact behavior.
They erase that difference completely on the calibration sample yet leave the frontier's residual on new data, an effect recently observed empirically \citep{jin2026iterative}.
The calibration a target requires can therefore be computed in advance: at $d=64$ and a signal-to-noise ratio of one per pair, a residual of $1/2$ needs $63$ calibration pairs.
For any other fitted subspace, \cref{cor:subspace-excess} adds an explicit excess: one step of iterative nullspace projection (INLP; \citealp{ravfogel2020inlp}) falls visibly short of the frontier, while fixed-rank R-LACE \citep{ravfogel2022rlace} reaches it.

\parhead{What kind of calibration helps.}
Real calibration comes in clusters, such as participants with many recordings each.
The guarantee depends only on the total noise of the averaged calibration difference, so more recordings from the same units stall at a cap that only more units lift, and \cref{cor:two-level} tells how to split a budget between the two.
On a public dataset in which $100$ people slept in a clinical polysomnograph and a wearable headband at once, the formula predicts, from calibration participants alone, how much device shift survives in new participants (\cref{sec:real-device}): even $128$ epochs from one participant leave about $0.7$ of the shift, whereas adding people keeps reducing it.

The limit stays exact for non-Gaussian shared content under pairing (\cref{sec:frontier}), the projection keeps its guarantee under anisotropic noise (\cref{sec:two-level}), and acting only after favorable fits cannot close the gap (\cref{sec:decisions}).
When the shift strength is unknown, certified rules decide when to act (\cref{sec:certification}).
Simulations, a semisynthetic experiment with real represented content, and two real device pairs confirm the predictions.
Together, these answers separate three reasons a correction can fall short: a suboptimal editor, too little calibration, or calibration collected on the wrong unit.

\section{Related work}
\label{sec:related}
\parhead{Detection versus direction.}
In spiked models a signal becomes detectable at a lower signal-to-noise ratio than its direction becomes estimable \citep{johnstone2009consistency,berthet2013optimal,cai2013sparse,mergny2026spectral}, and two-sample tests detect shift from modest samples \citep{rabanser2019failing}.
These results concern a test or an estimator.
We instead bound what any shared linear correction can remove under a distortion budget, which gives necessary as well as sufficient calibration requirements.
Our rules for when to act resemble selective prediction \citep{elyaniv2010selective} but select over calibration fits, and their certificates hold uniformly over unknown shift strengths, unlike risk control at a known distribution \citep{angelopoulos2025ltt}.

\parhead{Linear erasure from finite samples.}
Linear erasers fit a projection or affine map to data \citep{haghighatkhah2022mp,ravfogel2020inlp,ravfogel2022rlace,belrose2023leace,shao2023sal}, and related work studies what erasure preserves or leaves behind \citep{holstege2025splince,chowdhury2025perfect,akbari2025obliviator}.
Fitted on finite samples, they leave concept information on fresh data that population erasers would remove \citep{jin2026iterative}, and an estimated erasure subspace generalizes differently from the probe that tests it \citep{naowarat2026framework}.
\Cref{cor:empirical-erasers} gives this residual for MP, SAL and LEACE, whose binary-label equivalence \citet{belrose2023leace} noted, and \cref{cor:subspace-excess} gives the excess of any other fitted subspace.

\parhead{Paired and clustered calibration.}
Paired data isolate a shift from shared content, in travelling-subject harmonization \citep{yamashita2019harmonization}, simultaneous wearable and clinical EEG \citep{lopezlarraz2026boas,esparzaiaizzo2026sleep}, and counterfactual pairs, from which noisy counterfactual matching (NCM) and GRIT estimate a spurious subspace and bound the out-of-distribution error of a predictor restricted to its complement, GRIT at rate $\sqrt{d/N}$ in the number of pairs $N$ \citep{bai2025ncm,bai2026grit}.
These are upper bounds for a particular estimator and predictor, whereas our limit is exact and holds for every shared linear editor.
When pairs come in clusters, splitting a budget between participants and pairs is the two-stage allocation problem \citep{cochran1977sampling}, and the same allocation minimizes a bound on the residual of an eraser (\cref{cor:two-level}).

\section{Statistical experiment and loss}
\label{sec:model}
We formalize a shared linear correction learned from paired calibration and evaluated on independent fresh observations.
We assume isotropic within-source variation and calibration noise, so that no shift direction is easier to estimate or cheaper to delete than another.

Let $d\ge2$ and let $u\in\sphere$ be a unit vector in $\R^d$.
The unknown source shift is $\delta=\theta u$, where $\theta>0$.
A fresh representation is
\begin{equation}
 \begin{gathered}
 H=V+S\delta/2,\\
 V\sim N(0,\Id),\quad V\perp S,\quad S\in\{-1,+1\}.
 \end{gathered}
 \label{eq:fresh}
\end{equation}
An independent calibration sample consists of contrasts
\begin{equation}
 \begin{gathered}
 D_i\overset{\mathrm{iid}}\sim N(\delta,s^2\Id),\quad i=1,\ldots,N,\\
 X=\frac{\sqrt N\,\bar D}{s}\sim N(\sqrt\kappa\,u,\Id),
 \qquad \kappa=\frac{N\theta^2}{s^2}.
 \end{gathered}
 \label{eq:calibration}
\end{equation}
Here $\bar D=N^{-1}\sum_iD_i$ and $\kappa$ is the total calibration signal-to-noise ratio.
Initially $s>0$ is known; \cref{thm:unknown,app:unknown} address an unknown isotropic scale.
The contrasts can arise from same-content paired views with independent measurement noise (\cref{app:experiment}).

\parhead{What an editor may use.}
A rule chooses $M\in\R^{d\times d}$ from the calibration sample and an independent random seed.
The chosen matrix is then fixed, shared between sources, and known to the evaluator.
At known $s$, $X$ contains all information about $u$ in the full contrast sample.
Extra observations are allowed if their conditional law given $X$ is independent of $u$ (\cref{app:frontier}).

\parhead{A hard, source-neutral cost.}
The cost is how far the editor moves a fresh representation on average, and it must stay within budget on every fit and seed:
\begin{equation}
 C(M):=\E_V\|MV-V\|^2=\|M-\Id\|_F^2\le b,\qquad b\ge0.
 \label{eq:cost}
\end{equation}
Deleting one unit-variance direction costs one and deleting everything costs $d$.
The budget is total rather than per coordinate (average budgets: \cref{sec:selective}).

\parhead{A covariance-adjusted source residual.}
We score an editor by how much of the shift remains visible after editing, relative to the spread of the edited data.
The edited means differ by $M\delta$ and the within-source covariance is $MM^\top$.
Write $\Pi_K$ for the orthogonal projector onto a subspace $K$, and $\row(M),\ker M$ for the row space and kernel.
With $+$ the Moore--Penrose inverse, define
\begin{equation}
 L(M,u):=\frac{(M\delta)^\top(MM^\top)^+(M\delta)}{\theta^2}
 =\|\Pi_{\row(M)}u\|^2.
 \label{eq:loss}
\end{equation}
The numerator is squared Mahalanobis separation on the output support, so invertible rescaling leaves $L$ unchanged: a downstream classifier can undo any such rescaling, so shrinkage alone removes nothing.
By the second expression, $L\in[0,1]$, with $0$ denoting complete source removal and $1$ denoting no reduction of optimal Gaussian discriminability.
\Cref{app:experiment} gives the Gaussian KL and Bayes-error interpretations.
Because the residual and the cost are both covariance-adjusted, whitening turns an anisotropic within-source covariance into anisotropic calibration noise, which \cref{sec:two-level} treats.

The hard-budget minimax risk at fixed $\kappa$ is
\begin{equation}
 \risk_d^*(b,\kappa)
 :=\inf_{\widehat M:\,C(\widehat M)\le b\ \mathrm{pointwise}}
       \sup_{u\in\sphere}\E_u L(\widehat M,u).
 \label{eq:minimax}
\end{equation}
The expectation is over the calibration sample and the seed used to fit the shared editor.

\section{The exact cost of finite calibration}
\label{sec:frontier}
Under the loss \cref{eq:loss}, only the deleted subspace can help, and each deleted direction costs at least one.

\begin{lemma}[Repair is subspace estimation]
\label{lem:kernel}
For every $M\in\R^{d\times d}$ and $u\in\sphere$,
\begin{equation}
 L(M,u)=1-\|\Pi_{\ker M}u\|^2,
 \qquad
 C(M)\ge\dim\ker M,
 \label{eq:kernel}
\end{equation}
and $\Id-\Pi_{\ker M}$ has the same residual for every $u$ at cost exactly $\dim\ker M$ (\cref{app:frontier}).
\end{lemma}
The lemma covers asymmetric matrices, contractions, and nonorthogonal maps; a nonsingular shrinkage has $L=1$, however small its singular values.
It turns the search over all shared linear editors into the choice of a subspace to delete.
With one deletion, the residual is the squared sine of the angle between the shift and the deleted direction, so the question becomes how well the calibration points at it.
From there, the exact value follows from the posterior of a direction on the sphere and a randomization that equalizes the risk over directions.
The reduction is what makes the limit bind every editor rather than one estimator.

\begin{theorem}[Exact all-linear frontier]
\label{thm:frontier}
Set $k=\min\{d,\lfloor b\rfloor\}$ and let $J\sim\Pois(\kappa/2)$.
Then
\begin{equation}
 \risk_d^*(b,\kappa)=
 \begin{cases}
 1,&k=0,\\[2pt]
 (d-k)\displaystyle\E\frac1{d+2J},&1\le k\le d.
 \end{cases}
 \label{eq:frontier}
\end{equation}
For $k\ge1$, an attaining editor is $M=\Id-\Pi_K$, where $K=\operatorname{span}\{w\}\oplus K_\perp$, $w=\bar D/\|\bar D\|$, and $K_\perp$ is a uniformly sampled $(k-1)$-dimensional subspace of $w^\perp$.
It has constant risk over $u$ and cost $k$ on every fit.
\end{theorem}

\parhead{Why this is optimal.}
By \cref{eq:kernel}, the hard budget permits at most $k$ deleted directions.
Put a uniform prior on $u$.
Given $X$, the posterior of $u$ is a von Mises--Fisher law centred at $w$, with density proportional to $\exp(\sqrt\kappa\|X\|w^\top u)$.
It treats all directions orthogonal to $w$ alike, so its second-moment matrix has its largest eigenvalue along $w$ and equal eigenvalues on $w^\perp$.
The Bayes action therefore deletes $w$ and any $k-1$ orthogonal directions \citep{eaton1989invariance,mardia2000directional,fourdrinier2010spheres,besson2011subspace}.
Choosing those $k-1$ directions uniformly at random makes the frequentist risk the same for every $u$, so the Bayes risk is also the minimax risk.

The value has a simple form.
Rotate $u$ to the first coordinate, so that $X=(\sqrt\kappa+Z_1,Z_2,\ldots,Z_d)$ with $Z$ standard normal.
Deleting $w$ leaves the share of $\|X\|^2$ off the signal axis, $(Z_2^2+\dots+Z_d^2)/\|X\|^2$.
The numerator is $\chi^2_{d-1}$, and the signal coordinate $(\sqrt\kappa+Z_1)^2$ is an independent noncentral $\chi^2_1(\kappa)$.
The standard Poisson mixture for this ratio gives \cref{eq:frontier} \citep{marchand1997moments}.
\Cref{app:frontier} gives the details.

\parhead{What the formula says.}
For $1\le k\le d$ and $\kappa>0$, Jensen's and the Cauchy--Schwarz inequalities give (\cref{app:frontier})
\begin{equation}
 \frac{d-k}{d+\kappa}\le\risk_d^*(b,\kappa)
 \le\frac{d-k}{d-1}\cdot\frac{d}{d+\kappa}.
 \label{eq:bounds}
\end{equation}
For one deletion, the two bounds have an energy reading.
The lower bound $(d-1)/(d+\kappa)$ is the expected off-axis energy over the expected total.
The upper bound $d/(d+\kappa)$, the noise share of \cref{fig:concept}b, also counts the noise on the signal axis.
They differ by exactly $1/(d+\kappa)$.
Hence, at cost one, an expected residual at most $\varepsilon\in(0,1)$ requires $\kappa\ge(d-1)/\varepsilon-d$ and is guaranteed by $\kappa\ge d/\varepsilon-d$.
At a general budget it requires $k\ge d-\varepsilon(d+\kappa)$ deleted directions.

The ratio $\kappa/d$ governs the risk: as $d\to\infty$ with $\kappa=cd$ and fixed $k\ge1$, the risk tends to $1/(1+c)$, so the cost of directional uncertainty persists at constant signal per coordinate (other regimes: \cref{app:frontier}).
With known $u$, by contrast, cost one gives zero residual for every $\theta>0$.
Two features make the attaining editor robust.
Its upper bound uses only the mean and covariance of the contrasts, so it holds under non-Gaussian and anisotropic calibration noise at the mean noise level (\cref{prop:anisotropy}).
It uses only the direction of $\bar D$, so for $N\ge2$ it stays minimax when the noise scale is unknown, with the same risk $\risk_d^*(b,\kappa)$ at fixed $\kappa=N\theta^2/s^2$ (\cref{thm:unknown}).

\parhead{Detection and repair live at different scales.}
Testing $H_0:\delta=0$ against an unknown direction at strength $\kappa$ requires no estimate of that direction.
The level-$\alpha$ radial test rejects when $\|X\|^2>\chi^2_{d,1-\alpha}$, the $1-\alpha$ quantile of the central $\chi^2_d$ distribution.
It has exact power $\SF_{\chi^2_d(\kappa)}(\chi^2_{d,1-\alpha})$, with $\SF$ a survival function, and is minimax among level-$\alpha$ tests at a fixed radius.

\begin{corollary}[A detection--repair gap]
\label{cor:gap}
As $d\to\infty$ with fixed $\alpha\in(0,1)$, no level-$\alpha$ test has minimum power over directions exceeding $\alpha+o(1)$ when
$\kappa=o(\sqrt d)$; the radial test has power tending to one when
$\kappa/\sqrt d\to\infty$.
If simultaneously $\kappa/d\to0$, every fixed hard budget $b$ has minimax repair risk tending to one.
In that regime, achieving expected residual at most a fixed $\varepsilon<1$ requires deleting at least $(1-\varepsilon)d-o(d)$ directions.
\end{corollary}
The detection rate is classical \citep{baraud2002testing,collier2017minimax}, and the point is its contrast with \cref{eq:frontier}.

\parhead{Non-Gaussian shared content.}
The shared content of a pair need not be Gaussian.
Suppose each pair consists of shared content drawn from any fixed law with mean zero and covariance $(1-s^2/2)\Id$, plus independent Gaussian measurement noise $N(0,s^2\Id/2)$ on each endpoint, so that each view has identity covariance.
The contrast of the pair cancels the content, and its midpoint carries no information about the shift direction.
A shared affine editor $h\mapsto Mh+a$ then costs $\|M-\Id\|_F^2+\|a\|^2$, since a common offset adds cost but leaves the source separation unchanged.
This gives an exact extension: the whole frontier, with its calibration and deletion requirements, holds for every such content law, and the editor of \cref{thm:frontier} attains it without knowledge of the law (\cref{cor:real-content-frontier} in \cref{app:real-content-proof}).

\subsection{A consequence for empirical concept erasers}
\label{sec:empirical-erasers}
An empirical eraser removes exactly the mean difference it sees.
That difference is the true shift plus calibration noise, so the eraser deletes a noisy direction and leaves population separation behind.
For methods that delete the empirical direction, the frontier is this residual.

\begin{corollary}[Expected residual of one-direction empirical erasure]
\label{cor:empirical-erasers}
In the calibration experiment \cref{eq:calibration}, suppose a fitted shared affine editor $h\mapsto\widehat M h+\widehat a$ satisfies
$\widehat M\bar D=0$ and $\dim\ker\widehat M=1$ almost surely.
For every true direction $u$,
\begin{equation}
 \begin{aligned}
 L(\widehat M,u)&=1-\frac{(u^\top\bar D)^2}{\|\bar D\|^2},\\
 \E_u L(\widehat M,u)&=\risk_d^*(1,\kappa)
 =(d-1)\E_{J\sim\Pois(\kappa/2)}\frac1{d+2J}.
 \end{aligned}
 \label{eq:empirical-eraser-risk}
\end{equation}
\end{corollary}
\begin{proof}
Almost surely $\bar D\ne0$, so the two conditions give $\ker\widehat M=\operatorname{span}\{\bar D\}$; \cref{eq:kernel} gives the pointwise identity, and the cost-one case of \cref{thm:frontier} gives its expectation.
The common affine offset does not change source separation.
\end{proof}

For balanced paired observations, the empirical source means differ by $\bar D$.
MP \citep{haghighatkhah2022mp}, centered binary SAL with one direction removed \citep{shao2023sal} and exact binary LEACE \citep{belrose2023leace} all have kernel $\operatorname{span}\{\bar D\}$ almost surely, including LEACE with a singular empirical covariance (\cref{app:empirical-erasers}).
\Cref{eq:empirical-eraser-risk} is therefore their exact population residual, for every content law allowed by \cref{cor:real-content-frontier}, and \cref{eq:bounds} gives their calibration needs.
The three methods share the residual but not the distortion.
MP and SAL cost exactly one, LEACE's oblique map can cost more, and a mean-restoring offset adds to any of them.
Because their residual is known before fitting, \cref{eq:empirical-eraser-risk} can be inverted in $N$ at specified $(d,\theta,s)$, which gives the smallest calibration size $N_\star(\varepsilon)$ with expected residual at most $\varepsilon$.
At $d=64$ and $\theta=s^2=1$, the targets $1/2$ and $1/4$ require $N_\star=63$ and $190$ pairs (\cref{app:real-content-experiment}).
The excess for any other fitted subspace is the part of the empirical direction it leaves undeleted, weighted by how informative the calibration was.

\begin{corollary}[Excess residual of a fitted deleted subspace]
\label{cor:subspace-excess}
In \cref{eq:calibration}, fix $\kappa>0$ and an integer $1\le k\le d$, and let $K=\ker\widehat M$ satisfy $\dim K=k$ almost surely, where the fitted editor may use the full allowed calibration information and independent fitting randomness.
Write $\overline\E$ for expectation with $U$ uniform on $\sphere$ followed by calibration and fitting at direction $U$, and set $R=\|X\|$, $w=X/R$, and $\rho_K=\|(\Id-\Pi_K)w\|^2$, the fraction of the empirical direction outside $K$.
Then
\begin{equation}
 \begin{gathered}
 \overline\E L(\widehat M,U)
 =\risk_d^*(k,\kappa)
   +\overline\E\!\left[g_\kappa(R)\rho_K\right],\\
 g_\kappa(r)=1-d\,\lambda_\kappa(r)\ge0,
 \end{gathered}
 \label{eq:subspace-excess}
\end{equation}
where $\lambda_\kappa(r)=(1-\beta(\sqrt\kappa r))/(d-1)$ and $\beta(\sqrt\kappa r)=\overline\E[(w^\top U)^2\mid X]$ is the posterior second moment along $w$ in \cref{eq:posterior}.
The identity also holds for every content law $Q$ allowed by \cref{cor:real-content-frontier}.
\end{corollary}
\parhead{Proof (\cref{app:subspace-excess-proof}).}
Given the data, a kernel $K$ has expected residual $1-\tr(\Pi_KB_X)$, where $B_X$ is the posterior second-moment matrix \cref{eq:posterior}.

The first term in \cref{eq:subspace-excess} is the unavoidable direction-averaged residual at $k$ deleted directions.
The second vanishes whenever $K$ contains $w$.
Hence every fixed-$k$ rule that deletes $w$ has direction-averaged risk $\risk_d^*(k,\kappa)=(d-k)\,\risk_d^*(1,\kappa)/(d-1)$, however it chooses its other $k-1$ directions.
Each deletion beyond the empirical direction thus removes only $1/(d-1)$ of the residual left by the first, so a larger budget cannot substitute for more calibration.

\parhead{Application to iterative and adversarial erasers.}
An INLP iterate \citep{ravfogel2020inlp} or an R-LACE projection \citep{ravfogel2022rlace} at a fixed number of deletions is a fitted kernel, and \cref{cor:subspace-excess} applies to it whatever the method optimizes.
At equal kernel dimension, and hence equal orthogonal cost, it splits the residual into the statistical limit and the excess of the returned subspace, which gives a quantitative reading of the comparisons of INLP steps with random projections in \citet{haghighatkhah2022mp}.
Because $g_\kappa$ depends on $\kappa$, the excess is a diagnostic under the model, not a certificate computed from data alone.

\parhead{Controlled illustration.}
We fit MP, SAL, LEACE, fixed-step INLP, and a fixed-budget R-LACE variant to the same paired endpoints and evaluate their actual maps against the simulated shift (\cref{app:subspace-experiment}).
At $d=N=32$ and one deletion, one-step INLP leaves mean residual $0.687\pm0.014$, against $0.501\pm0.013$ for MP ($\pm1.96$ Monte Carlo SE).
Its predicted excess $0.181$ matches the observed paired excess $0.186\pm0.012$ over MP, and at four deletions the predicted $0.012$ matches $0.017\pm0.008$.
The fitted R-LACE subspaces nearly contain $w$ throughout the grid, and their mean residuals follow the fixed-rank frontier (\cref{fig:benchmark}a).

\parhead{Unpaired class samples.}
With $n_\pm$ unpaired Gaussian observations per source, the class-mean difference is $N(\delta,(1/n_++1/n_-)\Id)$ and is independent of the centered residuals, whose law is free of $u$.
The three erasers again delete it, so \cref{thm:frontier}, \cref{cor:empirical-erasers,cor:subspace-excess} and $N_\star$ apply with $\kappa=\theta^2/(1/n_++1/n_-)$.
For two classes of size $n$ this is $n\theta^2/2$, against $N\theta^2/s^2$ for $N$ pairs, so at equal sample size pairing multiplies $\kappa$ by $2/s^2$.
\Cref{cor:real-content-frontier}, by contrast, needs pairs to cancel the content.

\parhead{Correcting with the source label.}
A correction that knows the source can subtract an estimate $\hat\delta$ of the shift from one of them instead of deleting a direction from both.
Subtracting $\bar D$ itself leaves the share $\E\|\bar D-\delta\|^2/\theta^2=d/\kappa$, which exceeds one when $\kappa<d$; the shrunken $\frac{\kappa}{d+\kappa}\bar D$ leaves $d/(d+\kappa)$.
No estimate does much better: over all $\hat\delta$, even with $\theta$ and $s$ known, the minimax share left lies between $\risk_d^*(1,\kappa)$ and $d/(d+\kappa)$, which differ by at most $1/(d+\kappa)$ (\cref{prop:source-aware}).
The source label therefore saves the deleted direction, not the calibration: repair still needs $\kappa\asymp d$.

\section{Choosing when to act}
\label{sec:decisions}
\subsection{The oracle limit of selective repair}
\label{sec:selective}
The frontier averages over calibration fits, and some fits are more informative than others.
Can a rule beat the limit by correcting only after favorable fits?
We allow rules that, on each fit, either delete one direction $w$ at cost one or abstain and return the identity, and that act with probability at least $p$ at every shift direction.
Let $\risk_d^{\mathrm{sel}}(p,\kappa)$ be the smallest worst-case expected residual given that the rule acts, the risk--coverage criterion of selective prediction \citep{elyaniv2010selective} with coverage taken over calibration samples (\cref{app:selective}).

\begin{theorem}[Exact selective frontier]
\label{thm:selective}
Let $\tau_p$ satisfy $\SF_{\chi^2_d(\kappa)}(\tau_p)=p$, and let $J\sim\Pois(\kappa/2)$.
Then
\begin{equation}
 \risk_d^{\mathrm{sel}}(p,\kappa)
 =\frac{d-1}{p}\E_J\frac{\SF_{\chi^2_{d+2J}}(\tau_p)}{d+2J}.
 \label{eq:selective-frontier}
\end{equation}
It is attained by $w=X/\|X\|$ and $A=\{\|X\|^2\ge\tau_p\}$, which acts with probability exactly $p$ at every direction.
\end{theorem}
Acting on the largest calibration energies is optimal because the larger $\|X\|$, the more reliable the empirical direction: its posterior residual decreases with $\|X\|$.
It lowers the residual, but not enough: when $\kappa=o(d)$, the best conditional residual still tends to one at any action frequency bounded away from zero (\cref{cor:selective-gap}, \cref{app:selective}), so selection cannot close the detection--repair gap.
Under a budget that holds only on average over fits, the exact optimum instead spends complete deletion on the least informative fits, where one deletion barely helps (\cref{thm:average}, \cref{app:average}).

\subsection{Certified action without the shift strength}
\label{sec:certification}
The oracle threshold $\tau_p$ uses $\kappa$, which an observable rule does not know.
We instead require that the mean-direction editor $\Id-ww^\top$, $w=X/\|X\|$, applied on an event $A$, rarely both acts and misses a residual target $\varepsilon$, uniformly over all shifts:
\begin{equation}
 \sup_{\delta\ne0}\Prob_\delta\{A\ \text{and}\ L(\Id-ww^\top,\delta/\|\delta\|)>\varepsilon\}\le\alpha.
 \label{eq:joint}
\end{equation}
The guarantee is joint because the conditional error among accepted edits cannot be bounded uniformly: as $\delta\to0$, the residual of any action is $\operatorname{Beta}(\tfrac{d-1}2,\tfrac12)$-distributed whatever $\|X\|^2$ is.
A confidence ball around the calibration mean gives one valid rule (\cref{thm:certificate}), but it rarely acts: at $d=256$ and $\kappa=128$, where detection is near certain, it acts with probability about $10^{-8}$.
The Poisson representation behind the frontier gives a less conservative one.
Given the Poisson component, the calibration energy and the residual are independent, so controlling bad action within each component controls it for every $\kappa$ at once.
The resulting finite-mixture certificate acts strictly more often than the confidence ball at every shift strength (\cref{thm:envelope}), also with an estimated noise scale (\cref{app:unknown}).
As $\kappa/d\to c$, all these rules switch from almost never to almost always acting at $c=\varepsilon^{-1}-1$, exactly where the oracle conditional residual $1/(1+c)$ crosses the target (\cref{cor:phase}).

\section{Anisotropic and clustered calibration}
\label{sec:two-level}

Real calibration noise is rarely isotropic, and pairs often come in clusters, such as participants with several epochs each.
Clustering is a special case of anisotropic noise, and the guarantee of the attaining editor needs only two moments, so one proposition covers both.

\begin{proposition}[Bounds under anisotropic noise]
\label{prop:anisotropy}
Let the contrasts be i.i.d.\ with mean $\theta u$ and covariance $\Gamma$ in coordinates where the within-source covariance is $\Id$, with eigenvalues $\gamma_{(1)}\ge\dots\ge\gamma_{(d)}>0$, and set $\bar\kappa=N\theta^2d/\tr\Gamma$.
\begin{compactitems}
\item[(i)] Gaussian or not, the editor of \cref{thm:frontier} satisfies
$\E_uL\le\frac{d-k}{d-1}\cdot\frac{d}{d+\bar\kappa}$ for every $u\in\sphere$ and $1\le k\le d$.
\item[(ii)] If the contrasts are Gaussian, every editor within the hard budget $b$, with $k=\min\{d,\lfloor b\rfloor\}\ge1$, has
$\sup_u\E_uL\ge\max_{k<m\le d}\risk_m^*(k,N\theta^2/\gamma_{(m)})$,
whether or not $\Gamma$ is known.
\end{compactitems}
\end{proposition}
\parhead{Proof idea (\cref{app:anisotropy}).}
For (i), the Cauchy--Schwarz inequality bounds $\E(u^\top w)^2$ from below by $\bar\kappa/(d+\bar\kappa)$ using only the mean and covariance of $\bar D$.
For (ii), restricting $u$ to the $m$ noisiest directions leaves a problem no easier than the isotropic one at noise level $\gamma_{(m)}$.
For $\Gamma=s^2\Id$, part (i) is the upper bound in \cref{eq:bounds} and part (ii) at $m=d$ is \cref{thm:frontier}.

\parhead{Two-level calibration.}
Let $P$ participants contribute $n$ pairs each.
If participant $j$ has shift $\delta+\eta_j$ with $\Cov(\eta_j)=\Sigma_b$, and its pair contrasts vary around it with covariance $\Sigma_w$, the participant means are independent contrasts with mean $\delta$ and covariance $\Sigma_b+\Sigma_w/n$.
More pairs per participant shrink only the second term.

\begin{corollary}[Two-level calibration]
\label{cor:two-level}
In whitened coordinates, the mean-direction erasers have mean residual for the population shift $\delta$ at most $d/(d+\bar\kappa)$, with
\begin{equation}
 \bar\kappa(P,n)=\frac{P\theta^2d}{\tr(\Sigma_b+\Sigma_w/n)}\le\frac{P\theta^2d}{\tr\Sigma_b}\quad\text{for every }n.
 \label{eq:two-level}
\end{equation}
If a participant costs as much as $c$ pairs, then at a fixed total cost $P(c+n)$ the bound is smallest at $n^\star=\sqrt{c\,\tr\Sigma_w/\tr\Sigma_b}$ pairs per participant.
\end{corollary}
\parhead{Proof.}
Apply \cref{prop:anisotropy}(i) with $N=P$ and $\Gamma=\Sigma_b+\Sigma_w/n$; at fixed cost, $\bar\kappa$ is largest when $(\tr\Sigma_b+\tr\Sigma_w/n)(c+n)$ is smallest, which gives $n^\star$, the classical allocation of two-stage sampling \citep{cochran1977sampling}.

The calibration strength is therefore capped at $P\theta^2d/\tr\Sigma_b$ however many pairs each participant gives, and only more participants raise the cap.
The residual itself is the expectation of $1-(u^\top w)^2$ under the law of $\bar D$.
We evaluate it by plug-in Monte Carlo and call it the two-level prediction, to which the bound is close unless the noise spectrum is steep.

\section{Simulations and real device shifts}
\label{sec:experiments}
\begin{figure*}[t]
 \centering
 \includegraphics[width=\linewidth]{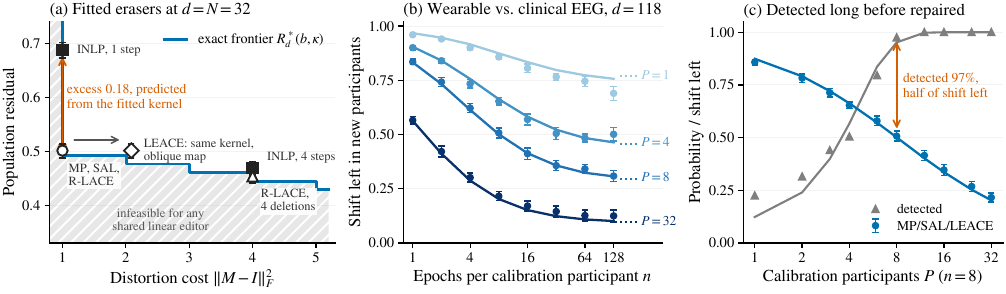}
 \caption{\textbf{What the frontier predicts for fitted erasers and for a real device shift.}
 (a) Simulation at $d=N=32$ with $\theta=s^2=1$: share of the shift left in the population against median distortion cost, over $256$ calibrations ($\pm1.96$ Monte Carlo SE). No shared linear editor can go below the exact frontier (line). Methods on the line are calibration-limited: four times the budget lowers the frontier only from $0.49$ to $0.44$. The orange arrow is the excess of one-step INLP predicted from its fitted kernel (\cref{eq:subspace-excess}), and the grey arrow is LEACE's extra distortion at the same kernel.
 (b, c) BOAS: share of the shift left in held-out participants (points, $\pm1.96$ SE over $20$ random splits of the same participants, which understates the uncertainty from the choice of people) and its prediction from the calibration pool alone (lines).
 (b) Dotted levels are the limits as $n\to\infty$.
 (c) At $n=8$, the shift is detected (grey) long before the erasers remove it (blue).}
 \label{fig:benchmark}
\end{figure*}

\subsection{Simulations and a semisynthetic check}
Simulations of the Gaussian experiment reproduce the exact formulas of \cref{sec:frontier,sec:decisions} within Monte Carlo error (\cref{tab:checks,fig:frontier} in \cref{app:numerics}).
For instance, the cost-one residual at $(d,\kappa)=(256,128)$ is $0.6644\pm0.0009$ against the exact $0.6652$.
Real represented content gives the same agreement, as \cref{cor:real-content-frontier} predicts.
We take three fixed pools of EEG features from a public paired-recording dataset \citep{williams2020flex}, inject shift and noise as in \cref{eq:real-content-pair}, and fit MP, SAL and LEACE.
Over $120$ conditions the residuals deviate from the predictions by $0.004$ in root mean square, and at the budgets $N_\star$ all three pools land on their targets (\cref{app:real-content-experiment}).
Selection and certification behave as \cref{sec:decisions} predicts.
At $d=64$ and $\kappa=32$, acting on only a fraction $10^{-5}$ of fits lowers the optimal conditional residual from $0.661$ to $0.549$ (\cref{fig:selective}).
Without certification, acting whenever the radial test rejects would, at $(256,128)$, act and miss the target $\varepsilon=0.5$ with probability $0.9986$.
The certified finite-mixture threshold is lower than the confidence ball's, $559.4$ against $588.6$ at $d=\kappa=256$ (\cref{fig:gate-comparison}).

\subsection{A wearable versus a clinical recording}
\label{sec:real-device}
In the BOAS dataset \citep{lopezlarraz2026boas}, $100$ healthy participants slept in a clinical polysomnograph and a forehead EEG headband at once.
For each participant we take one night, drawn with a fixed seed before any analysis.
We pair log-power spectra ($0.5$--$30$ Hz, $0.5$ Hz bins) of the two headband channels with those of the nearest polysomnograph channels in the same $30$-s N2 epoch, so each epoch gives one calibration pair and $d=118$.
A channel-quality rule leaves $85$ participants (\cref{app:boas}).
The participant is the calibration unit.
We split participants at random into a calibration pool of $43$ and a test set of $42$, whiten with the pool's within-source covariance, and estimate $\theta^2$, $\Sigma_b$ and $\Sigma_w$ from the pool alone.
We then draw $P$ calibration participants with $n$ epochs each, fit the erasers, and measure on the test participants how much of the population mean shift remains, over $20$ random splits.

The two-level prediction of \cref{sec:two-level} matches the observed residual across the grid $P\le32$, $n\le128$ (\cref{fig:benchmark}b): the root-mean-square deviation over $80$ cells is $0.016$, and the closed-form bound $d/(d+\bar\kappa)$ is never exceeded by more than $1.96$ standard errors.
The agreement holds without the channel-quality rule, for all sleep stages, for $d$ from $28$ to $236$, and for the actual MP, SAL and LEACE maps, whose kernels coincide (\cref{app:boas}).
On the $16$ participants of the Neuroscan/Flex N170 recordings \citep{williams2020flex}, a second device pair, the deviation is $0.005$ (\cref{app:real-pairs}).

The number of participants limits the correction.
All epochs of one participant share that participant's shift.
One participant caps $\bar\kappa$ at $\theta^2d/\tr\Sigma_b\approx35$, so even $128$ of its epochs leave $0.69$ of the shift, whereas four participants leave about half and $32$ leave $0.125$.
Because $\tr\Sigma_w/\tr\Sigma_b\approx11$, a participant who costs as much as $100$ epochs is best recorded for about $33$ epochs, or $17$ minutes of N2 sleep, a split that a pilot of eight nights already recovers (\cref{app:boas}).

Detection again comes first (\cref{fig:benchmark}c): with eight participants and eight epochs each, the shift is detected in $97\%$ of draws, yet half of it remains.
The gap widens with dimension, as \cref{cor:gap} leads one to expect: at $n=8$, half detection power needs about four participants for every $d$ from $28$ to $236$, while halving the shift needs about $4$, $6$, $8$ and $12$ participants at $d=28$, $58$, $118$ and $236$.
Subtracting the unshrunk mean difference with the known device label does worse than the erasers at every calibration size, and worse than no correction for fewer than $12$ participants at $n=8$, as its $d/\kappa$ share predicts.
Shrinkage removes this failure, but in the isotropic model no subtraction leaves less than the one-direction frontier (\cref{prop:source-aware}).

\section{Discussion}
\label{sec:discussion}
Within this model, the exact frontier is a benchmark for a correction that falls short.
As lower bounds, its limits also bind any larger model class containing this Gaussian experiment under the same constraints.
The frontier separates three reasons: the editor may be suboptimal (\cref{cor:subspace-excess}), the calibration may be too small for the budget (\cref{thm:frontier}), or the calibration may be collected on the wrong unit (\cref{cor:two-level}).
Evaluating representation repair should therefore report the achieved reduction, its distortion, and its frequency of application, rather than treating source detection as evidence that repair is already possible.

\parhead{Using the frontier in practice.}
The frontier plays the role of a power analysis for repair.
Before any calibration is collected, it gives the calibration signal-to-noise $\kappa$ that a target residual requires.
With repeated recordings per unit, such as participants or sites, three numbers fix $\kappa$: the size of the shift $\theta^2$, its variation between units $\tr\Sigma_b$, and the noise within a unit $\tr\Sigma_w$.
A small pilot estimates them, and \cref{eq:two-level} turns them into how many units to record and how many recordings per unit.
It also warns when more recordings from the same units cannot reach the target.
After fitting, the certificates of \cref{sec:decisions} decide, in the isotropic model, whether to apply the edit or to leave the data unchanged.

\parhead{Limitations.}
The exact results assume independent calibration, isotropic Gaussian contrasts, a single shift direction, shared linear or affine editing, and the covariance-adjusted residual and distortion.
With pairs, the frontier stays exact for any shared content (\cref{cor:real-content-frontier}).
Like every linear eraser, including LEACE, the editors remove only the first-moment difference.
Second-moment differences, which can be substantial between devices (\cref{app:boas}), are beyond any linear correction \citep{chowdhury2025perfect,akbari2025obliviator}.
All recordings are public and de-identified, and no new data were collected.\footnote{Code and saved results for every figure and table:\\\coderepo}
Neither source decodability nor a reduction in source separation alone establishes a downstream benefit: editing can remove useful information even when a residual target is met, so whether removal is desirable, in fairness, privacy, or prediction, must be evaluated in the application \citep{ravichander2021probing,elazar2021amnesic,kumar2022unreliable}.

\parhead{What remains open.}
Under anisotropic noise the minimax value is only bracketed (\cref{prop:anisotropy}), and weighting the direction estimate by the noise may beat the empirical one when the spectrum is steep.
Shifts spread over several directions, or more than two sources, call for a subspace version of the frontier.
Correcting one person means targeting their own shift rather than the population mean, which would combine both levels of \cref{eq:two-level}.
Calibrating until a certificate passes needs a sequential version of the guarantees.
Finally, the distortion cost treats all directions alike, whereas one that charges more for deleting task-relevant directions would trade the residual against downstream utility.
\label{sec:main-end}

\subsection*{AI use statement}
Generative artificial intelligence (AI) tools assisted code cleaning and refactoring, preliminary literature search, and manuscript text polishing.
All AI-assisted code was reviewed and tested by the authors.
The authors retain full responsibility for the final content.

\subsubsection*{Acknowledgements}
We would like to acknowledge support from the MBZUAI Startup Fund.

\bibliographystyle{apalike}
\bibliography{repair_v7_assets/references}

\clearpage
\appendix
\thispagestyle{empty}
\onecolumn
\aistatstitle{Detecting a Shift Is Not Enough:\\Exact Minimax Limits of Linear Representation Repair\\[2pt]Supplementary Materials}
\noindent\textbf{Roadmap.}
\Crefrange{app:experiment}{app:detection} prove the results of \cref{sec:model,sec:frontier}, \crefrange{app:selective}{app:phase} prove those of \cref{sec:decisions} on when to act, and \cref{app:anisotropy} proves \cref{prop:anisotropy} of \cref{sec:two-level}.
\Cref{app:numerics} reports the simulations of the theory, \cref{app:eraser-illustrations} fitted MP, SAL and LEACE on Gaussian and real content, \cref{app:boas} the wearable-versus-clinical EEG analysis with its Neuroscan/Flex replication, \cref{app:subspace-experiment} INLP and R-LACE at fixed rank, and \cref{app:anisotropy-check} a check under anisotropic content.

\section{Model details and the paired realization}
\label{app:experiment}
\subsection{A paired realization with unchanged content}
This appendix gives a paired construction that realizes the calibration experiment \cref{eq:calibration} and interprets the residual.
Both sources have positive probability, and their proportions do not enter the residual.

One realization of \cref{eq:calibration} uses, for $0<s^2<2$,
\begin{equation}
 H_i^+=C_i+\delta/2+E_i^+,
 \qquad H_i^-=C_i-\delta/2+E_i^-,
\end{equation}
where $C_i\sim N(0,(1-s^2/2)\Id)$ and
$E_i^+,E_i^-\overset{\mathrm{iid}}\sim N(0,s^2\Id/2)$ are independent.
Each view has within-source covariance $\Id$ and
$D_i=H_i^+-H_i^-\sim N(\delta,s^2\Id)$.
The midpoint is independent of the difference and its law does not involve $u$ at fixed $s$, so observing the full pair adds no directional information at known $s$.
Inference can therefore be based on the contrasts alone.
Any $s>0$ is allowed there, while the paired construction above needs $s^2<2$.
With $s=0$ a single contrast would identify $u$ and make cost-one repair exact, which is a different experiment.
When $s$ is unknown, observations whose law depends on $s$ could sharpen the scale estimate and with it the power of a certificate, and \cref{app:unknown} builds the estimated-scale certificates from the contrasts and their residuals.

\parhead{Gaussian interpretation of the residual.}
The numerator in \cref{eq:loss} is twice the KL divergence between the two edited Gaussian source distributions.
For equal source probabilities the Bayes classification error is $\Phi(-\theta\sqrt{L(M,u)}/2)$, where $\Phi$ is the standard normal distribution function.
In both readings the edited sources share the covariance $MM^\top$, and when $M$ deletes directions the two quantities are computed on its output support.

\subsection{Arbitrary shared content}
\label{app:real-content-proof}
Fix $d\ge2$, $\theta>0$, $0<s^2<2$, $\delta=\theta u$, and let the shared content $C$ have any fixed law $Q$ on $\R^d$ with
\begin{equation}
 \E_Q C=0,\qquad \Cov_Q(C)=(1-s^2/2)\Id,
 \label{eq:real-content-moments}
\end{equation}
where $Q$ does not depend on the unknown shift.
For each calibration pair, draw $C_i\sim Q$ and $E_i^+,E_i^-\sim N(0,s^2\Id/2)$, mutually independently and independently across pairs, and set
\begin{equation}
 H_i^{\pm}=C_i\pm\delta/2+E_i^{\pm}.
 \label{eq:real-content-pair}
\end{equation}
Fresh observations use independent draws $H=V+S\delta/2$ with $V=C+E$, where $C\sim Q$, $E\sim N(0,s^2\Id/2)$, and $S$ are independent.
The editor may know $Q$ and use the full calibration pairs and independent fitting randomness, but is fixed before observing fresh evaluation draws.

\begin{corollary}[The frontier with arbitrary shared content]
\label{cor:real-content-frontier}
Under \cref{eq:real-content-moments,eq:real-content-pair}, the source-neutral cost $\mathcal C_Q(M,a):=\E\|(M-\Id)V+a\|^2$ of a shared affine editor $h\mapsto Mh+a$ equals $\|M-\Id\|_F^2+\|a\|^2$, and its normalized covariance-adjusted squared mean separation $\mathcal L_Q(M,u):=(M\delta)^\top(MM^\top)^+(M\delta)/\theta^2$ equals $L(M,u)$ of \cref{eq:loss}.
For every hard budget $b\ge0$, the minimax expected loss over fitted editors $(\widehat M,\widehat a)$ satisfying $\mathcal C_Q(\widehat M,\widehat a)\le b$ on every fit and seed, with the worst case over $u\in\sphere$, equals $\risk_d^*(b,\kappa)$ with $\kappa=N\theta^2/s^2$.
The orthogonal editor of \cref{thm:frontier}, with zero offset, attains this value without knowledge of $Q$.
\end{corollary}
\begin{proof}
The fresh variation has $\E V=0$ and $\Cov(V)=\Cov(C)+\Cov(E)=\Id$, so expanding the quadratic distortion gives the cost identity; the edited source means differ by $M\delta$ with common covariance $MM^\top$, so \cref{eq:kernel} gives the loss identity; only the moments of $Q$ enter.

Next write each pair as its difference and midpoint,
$D_i=H_i^+-H_i^-=\delta+E_i^+-E_i^-$ and $A_i=(H_i^++H_i^-)/2=C_i+(E_i^++E_i^-)/2$.
The Gaussian noise difference and noise average have zero cross-covariance and hence are independent, and both are independent of $C_i$.
Thus $D_i\sim N(\delta,s^2\Id)$, and the midpoints are independent of the contrasts with a law that does not involve $u$; the same holds jointly across the $N$ pairs, and at fixed $s$ the centered contrast residuals are likewise independent of $\bar D$ with a $u$-free law.
Since the endpoints and $(D_i,A_i)_{i\le N}$ determine each other, all directional information in the full pairs is contained in $X$ of \cref{eq:calibration}; the remaining observations can serve only as fitting randomness, which \cref{thm:frontier} already allows.
Together with the two identities, this is the decision problem \cref{eq:minimax}, which gives the lower bound; an offset leaves the loss unchanged and adds nonnegative cost.
The attaining orthogonal editor depends only on $\bar D/\|\bar D\|$ and independent completion, so it has the same cost and risk here, proving equality.
\end{proof}

\section{Proof of Theorem~\lref{thm:frontier}}
\label{app:frontier}
\subsection{Reduction to subspace estimation}
For \cref{eq:kernel}, an SVD identifies $M^\top(MM^\top)^+M$ as the projector onto $(\ker M)^\perp$, and in an orthonormal basis beginning with $\ker M$ each kernel vector contributes one to $\|M-\Id\|_F^2$ while the remaining contributions are nonnegative.
Hence any matrix with kernel $K$ is weakly dominated in cost, at identical loss for all $u$, by $\Id-\Pi_K$.
If $\dim K<k$, enlarging $K$ to dimension $k$ weakly decreases the loss while remaining within budget.
It is therefore enough to consider rank-$k$ orthogonal projectors $\Pi_K$ in
\begin{equation}
 L=1-u^\top\Pi_Ku.
\end{equation}
When $k=0$, every feasible matrix is invertible and $L=1$.
When $k=d$, the zero editor attains zero loss.
The remaining case is $1\le k<d$.

By Gaussian sufficiency, the likelihood of the contrasts factors into the density of $X$ and a law of the centered residuals that does not depend on $u$.
Every variable whose conditional law given $X$ does not involve $u$, such as the residuals, can be simulated from $X$ and an independent seed.
It therefore suffices to consider randomized rules that use the data only through $X$.

\subsection{Posterior eigenspaces under a spherical prior}
Put a uniform prior on $u\in\sphere$.
Conditional on $X=x\ne0$, let $w=x/\|x\|$ and $a=\sqrt\kappa\|x\|$.
The posterior density is
\begin{equation}
 p(u\mid x)\propto \exp(a w^\top u).
\end{equation}
Rotations fixing $w$ leave this density invariant, so
\begin{equation}
 B_x:=\E[uu^\top\mid x]
 =\beta(a)ww^\top+\frac{1-\beta(a)}{d-1}(\Id-ww^\top),
 \qquad \beta(a)=\E[(w^\top u)^2\mid x].
 \label{eq:posterior}
\end{equation}
To see that $\beta(a)\ge1/d$, let $T=w^\top U$ for uniform $U$ on the sphere.
Symmetry gives
\begin{equation}
 \beta(a)=\frac{\E[T^2\cosh(aT)]}{\E[\cosh(aT)]}.
\end{equation}
Both $T^2$ and $\cosh(aT)$ are nondecreasing functions of $|T|$.
For any scalar random variable $Y$ and nondecreasing $f,g$,
$2\Cov(f(Y),g(Y))=\E[(f(Y)-f(Y'))(g(Y)-g(Y'))]\ge0$ with $Y'$ an independent copy.
Thus $\beta(a)\ge\E T^2=1/d$.
A rule that deletes $K$ captures the posterior expected signal $\tr(\Pi_K B_x)$.
The matrix $B_x$ has its largest eigenvalue $\beta(a)$ on $w$ and a common smaller eigenvalue on $w^\perp$, so any $K$ that contains $w$ maximizes this signal, and the remaining $k-1$ directions in $w^\perp$ can be chosen freely.

\subsection{From the Bayes action to a minimax action}
Choose the extra $k-1$ dimensional subspace uniformly in $w^\perp$, conditionally independently given $w$.
This rule is rotationally equivariant in distribution.
For every orthogonal $O$, the observation under parameter $Ou$ has the distribution of $OX$ under $u$, and the selected kernel transforms as $OK$.
Since $L(\Id-\Pi_{OK},Ou)=L(\Id-\Pi_K,u)$, the frequentist risk is constant on the sphere.
Its common risk equals the uniform-prior Bayes risk, which lower-bounds the maximum risk of every rule, including nonequivariant rules.
The rule therefore attains the minimax risk.
For $k=1$ no fitting randomization is needed.
Random completion is chosen once when fitting, and the evaluator observes the resulting matrix.
On the probability-zero event $\bar D=0$, any fixed $k$-dimensional $K$ may be chosen.

\subsection{Evaluation by a noncentral-beta moment}
It remains to evaluate the risk of the attaining rule.
Rotate $u$ to $e_1$.
For $Z\sim N(0,\Id)$ define
\begin{equation}
 U=(\sqrt\kappa+Z_1)^2\sim\chi^2_1(\kappa),\qquad
 E_\perp=\sum_{j=2}^d Z_j^2\sim\chi^2_{d-1},\qquad U\perp E_\perp.
\end{equation}
Deleting $w$ leaves loss $E_\perp/(U+E_\perp)$.
Given $w$, uniform completion captures the fraction $(k-1)/(d-1)$ of the remaining squared direction, in expectation.
Its expected residual is consequently
\begin{equation}
 \frac{d-k}{d-1}\E\frac{E_\perp}{U+E_\perp}.
\end{equation}
Use $U\mid J\sim\chi^2_{1+2J}$ with $J\sim\Pois(\kappa/2)$.
Conditional on $J$, the independent chi-square variables give
\begin{equation}
 \frac{E_\perp}{U+E_\perp}\mid J\sim\Beta\!\left(\frac{d-1}2,\frac{1+2J}2\right),
 \qquad
 \E\left[\frac{E_\perp}{U+E_\perp}\middle|J\right]=\frac{d-1}{d+2J}.
\end{equation}
This proves \cref{eq:frontier}.
The scalar expectation $\E[1/(d+2J)]$ is also the orthogonal coefficient of the exact isotropic projected-normal second moment derived by \citet{herrera2025projected}.
The same calculation shows that, given $J$, the ratio and total $U+E_\perp$ are independent, which proves \cref{eq:joint-moments}.

\subsection{Bounds and limiting regimes}
An integral form useful for both analysis and numerical cross-checking is
\begin{equation}
 f_d(\kappa):=\E\frac1{d+2J}
 =\frac12\int_0^1 t^{d/2-1}\exp\!\left(-\frac\kappa2(1-t)\right)\,dt.
 \label{eq:integral}
\end{equation}
Jensen's inequality gives $f_d(\kappa)\ge1/(d+\kappa)$, the lower bound in \cref{eq:bounds}.
For the upper bound, let $w=X/\|X\|$ and recall that the cost-one editor has expected residual $(d-1)f_d(\kappa)=1-\E(u^\top w)^2$.
By the Cauchy--Schwarz inequality, $(\E|u^\top X|)^2=(\E[|u^\top w|\,\|X\|])^2\le\E[(u^\top w)^2]\,\E\|X\|^2$; since $\E|u^\top X|\ge|\E u^\top X|=\sqrt\kappa$ and $\E\|X\|^2=d+\kappa$, we get $\E(u^\top w)^2\ge\kappa/(d+\kappa)$, that is, $(d-1)f_d(\kappa)\le d/(d+\kappa)$.
Multiplying by $(d-k)/(d-1)$ proves the upper bound in \cref{eq:bounds}; only the mean and second moment of $X$ were used.
Two further bounds are sharper at the extremes.
Monotonicity of the denominator gives $f_d(\kappa)\le1/d$, and since $d\ge2$, replacing $t^{d/2-1}$ by one gives
$f_d(\kappa)\le(1-e^{-\kappa/2})/\kappa\le1/\kappa$.
They improve on \cref{eq:bounds} only when $\kappa<d/(d-1)$ or $\kappa>d(d-1)$.
The low-signal limit follows from $J\to0$ in probability and bounded convergence.
For $\kappa=cd$ and $d\to\infty$, $2J/d\to c$ in probability, so
$d f_d(cd)=\E[1/(1+2J/d)]\to1/(1+c)$.
For fixed $d$ and $\kappa\to\infty$, substitute $z=\kappa(1-t)/2$ in \cref{eq:integral} and apply dominated convergence to obtain $\kappa f_d(\kappa)\to1$.

\parhead{Affine maps and alternative budgets.}
For an affine editor $h\mapsto Mh+a$, the common offset leaves source separation unchanged.
Its cost on centered $V$ is $\|M-\Id\|_F^2+\|a\|^2$, so it cannot improve the frontier.
If the cost were normalized by $d$, the same result would hold with $k$ replaced by $\min\{d,\lfloor db\rfloor\}$.
A fractional budget $b\in(k,k+1)$ cannot help.
By \cref{eq:kernel} the loss depends only on the kernel, and any map with a $(k+1)$-dimensional kernel costs at least $k+1$, whether it is orthogonal or oblique.
The same argument applies to a covariance-weighted distortion such as LEACE's, with the weighted norm in place of the Frobenius norm.
The exact solution when the budget is imposed only in expectation is \cref{thm:average}, proved in \cref{app:average}.

\subsection{Subtraction with a known source label}
\begin{proposition}[Source-aware subtraction]
\label{prop:source-aware}
In the experiment \cref{eq:calibration}, let a correction subtract from one source an estimate $\hat\delta$ of the shift that may depend on the calibration sample, an independent seed, $\theta$ and $s$.
With $\rho(a)=\E[w^\top u\mid x]$ under the uniform prior of \cref{eq:posterior},
\begin{equation}
 \risk_d^*(1,\kappa)\;\le\;
 \inf_{\hat\delta}\sup_{u\in\sphere}\E_u\frac{\|\delta-\hat\delta\|^2}{\theta^2}
 \;=\;1-\E\,\rho(\sqrt\kappa\|X\|)^2
 \;\le\;\frac{d}{d+\kappa},
 \label{eq:source-aware}
\end{equation}
and the two ends differ by at most $1/(d+\kappa)$.
The upper bound is the risk of $\hat\delta=\frac{\kappa}{d+\kappa}\bar D$.
\end{proposition}
\begin{proof}
Subtraction leaves the within-source covariance unchanged, so the share of squared separation left is $\|\delta-\hat\delta\|^2/\theta^2=\|u-v\|^2$ with $v=\hat\delta/\theta$.
The loss is convex in $v$ and $X$ is sufficient for $u$ at known $s$, so we may take $v$ to be a function of $X$.
Under the uniform prior, rotations fixing $w$ leave the posterior density invariant, so $\E[u\mid x]=\rho(a)w$, and the posterior expected loss of any $v$ is at least $1-\rho(a)^2$, with equality at $v=\rho(a)w$.
This rule is equivariant under rotations, so its risk does not depend on $u$ and equals its Bayes risk $1-\E\rho(\sqrt\kappa\|X\|)^2$; a Bayes rule with constant risk is minimax, which gives the equality.
Jensen's inequality gives $\rho(a)^2\le\beta(a)$.
Since $1-\E\beta(\sqrt\kappa\|X\|)=\risk_d^*(1,\kappa)$ is the Bayes risk of deleting $w$ in \cref{thm:frontier}, this proves the lower bound.
For the upper bound, $\bar D-\delta\sim N(0,s^2\Id/N)$ gives $\E\|\delta-c\bar D\|^2/\theta^2=(1-c)^2+c^2d/\kappa$, which equals $d/(d+\kappa)$ at $c=\kappa/(d+\kappa)$.
Finally, \cref{eq:bounds} gives $\risk_d^*(1,\kappa)\ge(d-1)/(d+\kappa)$.
\end{proof}

\subsection{Unknown noise scale}
\label{app:unknown-frontier}
\begin{corollary}[The hard-budget optimum with unknown noise scale]
\label{thm:unknown}
For $N\ge2$, fix $\kappa>0$ and parameterize $\delta=s\sqrt{\kappa/N}\,u$.
Even with $s>0$ unknown,
\begin{equation}
 \inf_{\widehat M:C(\widehat M)\le b}
 \sup_{s>0,\,u\in\sphere}\E_{s,u}L(\widehat M,u)
 =\risk_d^*(b,\kappa).
 \label{eq:adaptive}
\end{equation}
The editor in \cref{thm:frontier} attains this risk without knowing the shift strength or the noise scale.
\end{corollary}
\begin{proof}
Restrict the parameter space to any fixed $s=s_0$ and reveal $s_0$ to the procedure.
This can only lower the minimax risk, and the known-scale lower bound is \cref{eq:frontier}.
Conversely, the editor of \cref{thm:frontier} depends on the sample mean only through its direction, and its random completion uses no unknown parameter.
At each $s>0$, $\sqrt N\bar D/s\sim N(\sqrt\kappa u,\Id)$, so its risk equals \cref{eq:frontier} at every $(s,u)$.
This proves equality.
\end{proof}
The statement holds $\kappa$ fixed, so $\theta$ and $s$ vary together: with $\theta$ and $N$ fixed, more noise lowers $\kappa$ and raises the risk.

\section{Kernels of MP, SAL and LEACE (Corollary~\lref{cor:empirical-erasers})}
\label{app:empirical-erasers}
This section proves the method-specific conditions in \cref{cor:empirical-erasers}.
The projection and erasure formulas are existing constructions, and \cref{thm:frontier} supplies their finite-calibration population risk in the stated Gaussian experiment.
Computationally, MP and centered rank-one SAL need only the mean difference, which takes $O(Nd)$ time and $O(d)$ memory.
LEACE also needs the pooled covariance and its whitening, which take $O(Nd^2+d^3)$ time and $O(d^2)$ memory.
The action rules of \cref{sec:certification} need only $\|\bar D\|^2$ and a scalar threshold.

\subsection{Balanced paired moments}
Let $H_i^+,H_i^-$ be the $N$ pairs from \cref{app:experiment}, with source labels $+1,-1$ and differences $D_i=H_i^+-H_i^-$.
Write $\widehat\mu=(2N)^{-1}\sum_i(H_i^++H_i^-)$ and define the centered pooled moments
\begin{align}
 \widehat\Sigma
 &=\frac1{2N}\sum_{i=1}^N\sum_{z\in\{-1,+1\}}
 (H_i^z-\widehat\mu)(H_i^z-\widehat\mu)^\top,\\
 v&=\frac1{2N}\sum_{i=1}^N\sum_{z\in\{-1,+1\}}
 (H_i^z-\widehat\mu)z=\frac{\bar D}{2}.
 \label{eq:empirical-moments}
\end{align}
Both source endpoints are observed for each calibration pair, so this experiment has $N$ pairs and $2N$ labeled representations.
The full-pair construction keeps the within-source fresh covariance equal to $\Id$ and makes the midpoint ancillary for the direction at fixed noise scale.
Consequently these fitted editors are also admissible with respect to the information restriction in \cref{sec:model}, although their cost must still be checked separately.
The residual identity itself only requires the stated kernel and the distribution of $\bar D$, regardless of how other fitted quantities are obtained.

\subsection{MP and centered rank-one SAL}
For MP \citep{haghighatkhah2022mp}, let $w=\bar D/\|\bar D\|$ and set $M_{\rm MP}=\Id-ww^\top$.
Its kernel is $\operatorname{span}\{\bar D\}$.
For SAL \citep{shao2023sal}, use the centered representation--label cross-covariance and remove its nonzero left singular direction.
With a scalar binary label this cross-covariance is $v$.
With centered two-class one-hot labels its two columns are opposite multiples of $v$.
In either case it has rank one almost surely, and the retained map is again $\Id-ww^\top$.
The statement concerns centered binary SAL that removes exactly one direction.
Other variants are covered by \cref{cor:subspace-excess}, which applies to any fitted kernel.
The scalar-binary equivalence of MP and SAL is also identified in Appendix D of \citet{belrose2023leace}.
If either map is applied about the fitted pooled mean, its offset is $(\Id-M)\widehat\mu$.
The offset leaves the residual unchanged but adds to the distortion.

\subsection{Binary LEACE, including singular empirical covariance}
Let $Y$ be the matrix whose columns are the centered endpoints.
Then $\widehat\Sigma=YY^\top/(2N)$ and $v=Yz/(2N)$ for the vector of source labels $z$.
Therefore $v\in\operatorname{range}(Y)=\operatorname{range}(\widehat\Sigma)$.
Since $v\ne0$ almost surely, $v^\top\widehat\Sigma^+v>0$ even if $\widehat\Sigma$ is singular.
The exact binary LEACE map can be written
\begin{equation}
 q=\frac{\widehat\Sigma^+v}{v^\top\widehat\Sigma^+v},
 \qquad M_{\rm LEACE}=\Id-vq^\top,
 \qquad a_{\rm LEACE}=vq^\top\widehat\mu.
 \label{eq:empirical-leace}
\end{equation}
This is the usual whiten--project--unwhiten expression \citep{belrose2023leace}: with $W=(\widehat\Sigma^+)^{1/2}$, its removed map is
$\widehat\Sigma^{1/2}(Wv)(Wv)^\top W/\|Wv\|^2=vq^\top$, using $v\in\operatorname{range}(\widehat\Sigma)$.
As $q^\top v=1$, we have $M_{\rm LEACE}v=0$.
Conversely, $M_{\rm LEACE}x=0$ implies $x=v(q^\top x)$, so its kernel is exactly $\operatorname{span}\{v\}$.
The covariance and offset can vary across fits without changing this kernel or the residual in \cref{eq:empirical-eraser-risk}.

The linear and affine reference costs are, respectively,
\begin{equation}
 C(M_{\rm LEACE})=\|v\|^2\|q\|^2\ge1,
 \qquad
 C(M_{\rm LEACE},a_{\rm LEACE})
 =\|v\|^2\bigl(\|q\|^2+(q^\top\widehat\mu)^2\bigr).
 \label{eq:empirical-leace-cost}
\end{equation}
The inequality follows from $q^\top v=1$ and Cauchy--Schwarz.
The empirical LEACE distortion objective uses its fitted data moments and need not coincide with the population reference cost \cref{eq:cost}.
Neither an observed cost on one fit nor its empirical average establishes the deterministic budget required by \cref{thm:frontier}.

\subsection{Residual law and certificates}
By \cref{cor:empirical-erasers} these erasers have the same residual on each calibration sample, and the calculation in \cref{app:frontier} gives its law, $L\mid J\sim\Beta(\tfrac{d-1}2,\tfrac{1+2J}2)$ with $J\sim\Pois(\kappa/2)$.
Hence a certificate for the mean-direction projection computed from the same contrasts (\cref{app:certificates,app:envelope,app:unknown}) controls the identical residual of these erasers under \cref{eq:joint}.
Their method-specific distortion is not covered by the hard budget.

\subsection{Proof of Corollary~\lref{cor:subspace-excess}}
\label{app:subspace-excess-proof}
Let $\mathcal F$ include all calibration observations and the fitting seed.
By the factorization of the experiment, the conditional law of $U$ given $\mathcal F$ equals that given $X$, whose second-moment matrix is $B_X=\lambda_\kappa(R)\Id+g_\kappa(R)ww^\top$ by \cref{eq:posterior}.
Since $\tr(\Pi_K)=k$ and $\tr(B_X)=1$,
\begin{equation}
 \overline\E[L(\widehat M,U)\mid\mathcal F]
 =1-\tr(\Pi_K B_X)
 =(d-k)\lambda_\kappa(R)+g_\kappa(R)\rho_K.
 \label{eq:subspace-conditional-risk}
\end{equation}
The posterior ordering in \cref{eq:posterior} gives $g_\kappa\ge0$, and the frontier calculation gives
$\overline\E\lambda_\kappa(R)=\E_{J\sim\Pois(\kappa/2)}[1/(d+2J)]$; taking expectations and using \cref{eq:frontier} proves \cref{eq:subspace-excess}.
The full-pair extension has the same posterior by the midpoint--contrast factorization in \cref{app:real-content-proof}.
\qed

\section{Proof of Corollary~\lref{cor:gap}}
\label{app:detection}
\subsection{The radial test is minimax at a fixed radius}
Let $\pi$ be uniform on the sphere and $P_\pi=\int N(\sqrt\kappa u,\Id)\,\pi(du)$ its Gaussian mixture.
Relative to $P_0=N(0,\Id)$, its likelihood ratio is
\begin{equation}
 \Lambda(x)=e^{-\kappa/2}\E_U e^{\sqrt\kappa x^\top U}.
\end{equation}
By symmetry the last expectation equals $\E\cosh(\sqrt\kappa r U_1)$ with $r=\|x\|$.
It therefore depends only on $r$ and increases in $r\ge0$, because $\cosh$ is even and increasing on $[0,\infty)$.
The Neyman--Pearson test for this mixture rejects at large $\|x\|^2$.
It maximizes average power under the uniform mixture and has equal power at every direction.
Average power upper-bounds minimum power for every competing test, proving the minimax claim.

\subsection{Lower and upper detection scales}
For independent uniform $U,U'$,
\begin{equation}
 1+\chi^2(P_\pi\|P_0)=\E_{U,U'}e^{\kappa U^\top U'}
 \le\exp(\kappa^2/(2d)).
 \label{eq:chi-detection}
\end{equation}
For completeness, $U^\top U'$ has the law of the first coordinate of a uniform sphere point and
\begin{equation}
 \E U_1^{2m}=\frac{(2m-1)!!}{d(d+2)\cdots(d+2m-2)}
 \le\frac{(2m-1)!!}{d^m};
\end{equation}
expanding the exponential proves \cref{eq:chi-detection}.
If $\kappa=o(\sqrt d)$, the total variation distance tends to zero by the chi-square bound.
No level-$\alpha$ test can then have minimum power above $\alpha+o(1)$.

Under the alternative, $\|X\|^2$ has mean $d+\kappa$ and variance $2d+4\kappa$.
For fixed $\alpha$, its null critical value is $d+O(\sqrt d)$.
If $\kappa/\sqrt d\to\infty$, the excess mean over this threshold dominates the alternative standard deviation, so Chebyshev's inequality gives power tending to one.
Finally, for fixed $k\ge1$ and $\kappa=o(d)$,
$1\ge\risk_d^*\ge(d-k)/(d+\kappa)\to1$.
For $k=0$ the risk equals one at every $\kappa$.
Rearranging the same lower bound gives the deletion requirement in \cref{cor:gap}.

\section{Proof of Theorem~\lref{thm:selective}}
\label{app:selective}
\begin{corollary}[Abstention cannot close the gap at a fixed frequency]
\label{cor:selective-gap}
For every $p\in(0,1)$,
\begin{equation}
 \risk_d^{\mathrm{sel}}(p,\kappa)
 \ge\max\left\{0,\ 1-\frac{1-\risk_d^*(1,\kappa)}p\right\}.
 \label{eq:selective-bound}
\end{equation}
If $\kappa=o(d)$ and $p$ is bounded below by a positive constant, the best conditional residual tends to one.
\end{corollary}
The attaining selective rule, returning the identity on abstention, has unconditional risk $1-p+p\risk_d^{\mathrm{sel}}(p,\kappa)\ge\risk_d^*(1,\kappa)$, which rearranges to \cref{eq:selective-bound}.
Thus the detection--repair gap persists when a rule may choose when to act but must keep its action frequency bounded away from zero as dimension grows.
The frontier extends continuously to $p=1$ by $\risk_d^{\mathrm{sel}}(1,\kappa)=\risk_d^*(1,\kappa)$.

Keep $\kappa>0$ fixed and use the uniform prior $\pi$ over $u$.
By \cref{eq:posterior}, the minimum posterior residual among rank-one deletions is
$h_\kappa(r)=1-\beta(\sqrt\kappa r)$ at $r=\|X\|$, attained by deleting $X/\|X\|$.
It is strictly decreasing for $r>0$.
To prove this, set $Y=|U_1|$ for a uniform sphere point and let $\E_a$ denote expectation under the density tilted by $\cosh(aY)$.
Differentiation under the bounded integral gives
\begin{equation}
 \beta'(a)=\Cov_a\{Y^2,Y\tanh(aY)\}>0\qquad(a>0).
\end{equation}
Both $Y^2$ and $Y\tanh(aY)$ are strictly increasing in $Y$.
Since $Y$ is nondegenerate, the covariance argument of \cref{app:frontier} shows that their covariance is strictly positive.

For any feasible rule, set $q(u)=\Prob_u(A)$, $n(u)=\E_u[\ind_A L]$, and
$q=\int q(u)\,d\pi(u)\ge p$.
Since $q(u)\ge p>0$, the worst conditional residual obeys
\begin{equation}
 \sup_u\frac{n(u)}{q(u)}
 \ge\frac{\int n(u)\,d\pi(u)}{\int q(u)\,d\pi(u)}
 \ge\frac{\E_\pi[\ind_A h_\kappa(\|X\|)]}{q}.
 \label{eq:selective-average}
\end{equation}
The second inequality conditions on the fit, including its random seed: none of those extra variables changes the posterior of $u$ given $X$ under the experiment's factorization assumption.
The selected direction therefore cannot have posterior residual below $h_\kappa$.

For a fixed average action probability $q$, the last numerator is minimized by taking the largest norms.
For an explicit argument, let $B_q=\ind\{\|X\|\ge r_q\}$ have probability $q$ under the prior predictive law, and write $h_q=h_\kappa(r_q)$.
Pointwise,
\begin{equation}
 (\ind_A-B_q)\{h_\kappa(\|X\|)-h_q\}\ge0.
\end{equation}
Its expectation and $\E(\ind_A-B_q)=0$ imply
$\E[\ind_A h_\kappa]\ge\E[B_qh_\kappa]$.
This remains valid for randomized actions by conditioning on $X$.
The tail average $m(q)=\E[B_qh_\kappa]/q$ is nondecreasing in $q$, since enlarging the accepted tail adds only larger values of $h_\kappa$.
Therefore \cref{eq:selective-average} is at least $m(p)$.

The mean-direction rule with $A=\{\|X\|^2\ge\tau_p\}$ is rotationally equivariant, so it has probability $p$ and the same conditional residual at every $u$.
It attains the posterior bound and the tail rearrangement bound, hence attains $m(p)$.
The ratio of the first-moment and probability identities in \cref{eq:joint-moments} gives \cref{eq:selective-frontier}.
This proves \cref{thm:selective}.
\Cref{cor:selective-gap} follows from applying \cref{thm:frontier} to this rule including its identity action on abstention.

As $p\uparrow1$, dominated convergence recovers the ungated rank-one risk.
As $p\downarrow0$, $r_p\to\infty$ and the posterior density concentrates on the mean direction, so $h_\kappa(r_p)\to0$.
Since $0\le m(p)\le h_\kappa(r_p)$, the selective frontier tends to zero at fixed $d,\kappa>0$.
The latter limit permits arbitrarily rare actions and does not contradict the fixed-frequency asymptotic lower bound.

\section{The average-budget frontier}
\label{app:average}
An average budget $\sup_u\E_u C(\widehat M)\le b$ lets the distortion depend on how informative the calibration fit is.
At known scale and oracle-known $\kappa$, define
\begin{equation}
 \risk_d^{\rm av}(b,\kappa)
 =\inf_{\widehat M:\,\sup_u\E_u C(\widehat M)\le b}
   \sup_u\E_u L(\widehat M,u).
 \label{eq:average-definition}
\end{equation}
The expectation in the constraint is over calibration and fitting randomness.

\begin{theorem}[Exact average-budget frontier]
\label{thm:average}
For known scale, fixed $\kappa>0$, and $0<b<d$,
\begin{equation}
 \risk_d^{\rm av}(b,\kappa)=
 \begin{cases}
  1-b+b\risk_d^{\rm sel}(b,\kappa),&0<b<1,\\
  \risk_d^*(1,\kappa),&b=1,\\
  q\risk_d^{\rm sel}(q,\kappa),\quad q=(d-b)/(d-1),&1<b<d.
 \end{cases}
 \label{eq:average-frontier}
\end{equation}
The endpoint risks are one at $b=0$ and zero at $b\ge d$.
For $b<1$, delete the mean direction on the largest-norm tail of probability $b$ and otherwise return the identity.
For $b>1$, delete the entire representation on the smallest-norm tail of probability $(b-1)/(d-1)$, and otherwise delete only the mean direction.
These rules have expected cost exactly $b$ at every $u$.
\end{theorem}
The attaining rules use the known scale and oracle knowledge of $\kappa$.

\subsection{Proof and the optimal allocation}
Let $B_X=\E_\pi[uu^\top\mid X]$ under the uniform-sphere prior.
For any fitted $M$, the factorization assumption gives $\E_\pi[1-L(M,u)\mid X,M]=\tr(\Pi_{\ker M}B_X)$, and the cost of $M$ is at least $\dim\ker M$.
For every $\lambda\ge0$, therefore,
\begin{equation}
 \tr(\Pi_{\ker M}B_X)-\lambda C(M)
 \le\tr(\Pi_{\ker M}(B_X-\lambda\Id))
 \le\tr(B_X-\lambda\Id)_+,
 \label{eq:average-dual}
\end{equation}
where $(\cdot)_+$ replaces negative eigenvalues by zero.
Integrating the expected-cost constraint against $\pi$ gives $\E_\pi C(M)\le b$, so every feasible rule satisfies
\begin{equation}
 \sup_u\E_u L(M,u)
 \ge1-\lambda b-\E_\pi\tr(B_X-\lambda\Id)_+.
 \label{eq:average-lower}
\end{equation}
This lower bound applies directly to all matrices and all fitting randomization.

Write $r=\|X\|$, $\beta_r=\beta(\sqrt\kappa r)$, and
$\gamma_r=(1-\beta_r)/(d-1)$.
The eigenvalues of $B_X$ are $\beta_r$ on the mean direction and $\gamma_r$ with multiplicity $d-1$ on its orthogonal complement.
For $r>0$ and $\kappa>0$, $\beta_r>1/d>\gamma_r$; $\beta_r$ strictly increases from $1/d$ to one while $\gamma_r$ strictly decreases from $1/d$ to zero.
The distribution of $r$ is continuous and does not depend on $u$.

If $0<b<1$, choose $r_b$ so that $\Prob(r\ge r_b)=b$ and take
$\lambda=\beta_{r_b}>1/d$.
The positive eigenspace of $B_X-\lambda\Id$ is empty below $r_b$ and consists only of the mean direction above $r_b$.
Thus the stated identity-or-rank-one rule attains equality in \cref{eq:average-dual} and has expected cost $b$.
By rotational equivariance its cost and risk are constant over $u$, attaining the lower bound \cref{eq:average-lower}.
Its risk is $1-b+b\risk_d^{\rm sel}(b,\kappa)$.

If $1<b<d$, set $q=(d-b)/(d-1)$ and choose $r_q$ with $\Prob(r\ge r_q)=q$.
Take $\lambda=\gamma_{r_q}<1/d$.
The mean direction always has a positive coefficient, while all $d-1$ remaining directions have positive coefficients exactly when $r<r_q$.
The optimal kernel is therefore the entire space below $r_q$ and the mean direction above it.
Expected cost is $d(1-q)+q=b$, and the risk is
$\E[L_1\ind\{r\ge r_q\}]=q\risk_d^{\rm sel}(q,\kappa)$.
Again both are constant over $u$, and equality in the lower bound proves optimality.
At $b=1$, choose $\lambda=1/d$ and always delete only the mean direction.
At $b=0$ only the identity has zero cost, so the risk is one.
At $b\ge d$ the zero matrix is affordable and removes the whole shift.

The optimum therefore differs from mixing adjacent hard-budget ranks, which would randomize the deletion rank independently of the fit.
Above cost one it spends the additional deletions on the least informative, low-norm fits (\cref{fig:average}).
At $\kappa\downarrow0$ the posterior is isotropic and the frontier tends to $1-b/d$, the value of random deletion.

\section{The confidence-ball rule and exact action laws}
\label{app:certificates}
\begin{theorem}[Confidence-ball action rule]
\label{thm:certificate}
Let $\alpha,\varepsilon\in(0,1)$, $c_\alpha=\chi^2_{d,1-\alpha}$, and act exactly when
\begin{equation}
 A=\left\{\|X\|^2\ge\frac{c_\alpha}{\varepsilon}\right\}.
 \label{eq:gate}
\end{equation}
The mean-direction editor satisfies \cref{eq:joint} for every nonzero $\delta$.
At $\delta=0$, its probability of taking action is at most $\alpha$.
For nonzero $\delta$, its exact action probability is
$\SF_{\chi^2_d(\kappa)}(c_\alpha/\varepsilon)$.
\end{theorem}
The ball's tangent cone is the classical directional confidence cone \citep{box1987empirical}, and its angular width bounds the error of the estimated direction.

\subsection{Geometry of a confidence ball}
Since $X-\sqrt N\delta/s\sim N(0,\Id)$, the event
$\mathcal C=\{\|X-\sqrt N\delta/s\|^2\le c_\alpha\}$ has probability exactly $1-\alpha$ at every shift.
Consider a ball with center $x\ne0$ and radius $r<\|x\|$.
For any nonzero $z$ inside the ball, the line spanned by $z$ intersects the ball, so its distance to $x$ is at most $r$.
That distance squared is $\|x\|^2\{1-(x^\top z)^2/(\|x\|^2\|z\|^2)\}$.
Therefore the largest squared sine from the center direction is at most $r^2/\|x\|^2$; equality is attained at a tangent direction.
Taking $x=X$ and $r^2=c_\alpha$ proves \cref{thm:certificate} on $\mathcal C$.
On the action event $r^2/\|x\|^2\le\varepsilon<1$, so the ball does not contain zero.
A bad action can therefore occur only when the ball misses the true shift, and its joint probability is at most $\alpha$ at every nonzero shift.
That $w$ depends on the data does not affect this argument, because coverage is a statement about the whole ball.
At the null, the action event is contained in $\{\|X\|^2\ge c_\alpha\}$ and therefore has probability at most $\alpha$.
The noncentral chi-square law of $\|X\|^2$ gives the exact action probability stated in the theorem.

\subsection{Exact laws for a radial action rule}
For the mean-direction editor $w=X/\|X\|$ and any $t\ge0$, set $A_t=\{\|X\|^2\ge t\}$, write $p_t=\Prob(A_t)$, and let $L_1=1-(u^\top w)^2$.
Under a fixed direction, the variables $U,E_\perp,J$ of \cref{app:frontier} give $T_0=U+E_\perp\mid J\sim\chi^2_{d+2J}$ and $L_1=E_\perp/(U+E_\perp)\mid J\sim\Beta((d-1)/2,(1+2J)/2)$, independently given $J$; conditioning on $J$ and applying the event $T_0\ge t$ therefore gives, with $a_J=(d-1)/(d+2J)$,
\begin{align}
 p_t&=\E_J\SF_{\chi^2_{d+2J}}(t),\nonumber\\
 \E[L_1\ind_{A_t}]&=\E_J\left[a_J\SF_{\chi^2_{d+2J}}(t)\right],\label{eq:joint-moments}\\
 \Prob(A_t,L_1>\varepsilon)&=\E_J\left[\SF_{\chi^2_{d+2J}}(t)
 \SF_{\Beta((d-1)/2,(1+2J)/2)}(\varepsilon)\right].\nonumber
\end{align}
Thus the mean residual among actions is $\E[L_1\ind_{A_t}]/p_t$.
Counting abstention as the identity, the total expected residual is $1-p_t+\E[L_1\ind_{A_t}]$.
The ball rule can be highly conservative: for $d=256$, $\kappa=128$, and $(\alpha,\varepsilon)=(0.05,0.5)$, its action probability is about $10^{-8}$ despite near-certain detection.
The probability of failure \emph{among} actions is
\begin{equation}
 \Prob(L_1>\varepsilon\mid A_t)
 =\frac{\Prob(A_t,L_1>\varepsilon)}{p_t},
\end{equation}
which is different from the unconditional joint probability in \cref{eq:joint}.
At $d=64$ and $\kappa=32$, the confidence-ball rule acts with probability $8.3\times10^{-5}$, yet its conditional mean residual is $0.56$, above the target $0.5$.
Controlling the joint bad-action probability therefore does not guarantee the target residual among accepted corrections.
All action probabilities are positive for finite thresholds in this Gaussian model, though they may be extremely small.
The guarantee assumes a prespecified $N,\alpha,\varepsilon$ and a single decision.
Repeatedly testing growing calibration samples, or choosing the reported target after inspecting the data, needs an additional simultaneous or sequential argument.

\section{The finite-mixture rule}
\label{app:envelope}
This section lowers the confidence-ball threshold by controlling the bad-action probability separately in each Poisson component of the noncentral chi-square law.
For $j\in\mathbb Z_{\ge0}$, define
\begin{equation}
 \ell_j=\SF_{\Beta((d-1)/2,(1+2j)/2)}(\varepsilon),
 \qquad \mathcal J=\{j\in\mathbb Z_{\ge0}:\ell_j>\alpha\}.
 \label{eq:component-tails}
\end{equation}
The sequence $\ell_j$ decreases to zero, so $\mathcal J$ is finite.
Write $Q_{m}(z)=\chi^2_{m,1-z}$ for the upper-tail quantile, and take a maximum over an empty set as zero.

\begin{theorem}[Finite-mixture action certificate]
\label{thm:envelope}
For $\alpha,\varepsilon\in(0,1)$, let
\begin{equation}
 t_{\rm mix}=\max\left\{c_\alpha,\ \max_{j\in\mathcal J}Q_{d+2j}(\alpha/\ell_j)\right\}.
 \label{eq:envelope}
\end{equation}
Act with the mean-direction editor when $\|X\|^2\ge t_{\rm mix}$.
For every nonzero $\delta$, the joint bad-action probability in \cref{eq:joint} is at most $\alpha$; null action is also at most $\alpha$.
Moreover,
\begin{equation}
 t_{\rm mix}\le\max\{c_\alpha,\chi^2_{d-1,1-\alpha}/\varepsilon\}
 <c_\alpha/\varepsilon.
 \label{eq:strict-improvement}
\end{equation}
Thus its action probability is strictly larger than the confidence-ball rule's at every finite $\kappa\ge0$.
\end{theorem}

\subsection{A finite family of componentwise constraints}
Condition on the exact Poisson representation from \cref{app:frontier}.
Given $J=j$, let $U_j\sim\chi^2_{1+2j}$ and $E_\perp\sim\chi^2_{d-1}$ be independent.
The total $S_j=U_j+E_\perp\sim\chi^2_{d+2j}$ is independent of
$L_j=E_\perp/(U_j+E_\perp)\sim\Beta((d-1)/2,(1+2j)/2)$.
Consequently the conditional joint bad-action probability at threshold $t$ is
\begin{equation}
 g_j(t)=\ell_j\SF_{\chi^2_{d+2j}}(t).
 \label{eq:component-error}
\end{equation}
The actual probability at any fixed noncentrality is
$\E_{J\sim\Pois(\kappa/2)}g_J(t)$.
This is the exact law at the true $\kappa$ and involves no prior on it.

To show that $\ell_j$ decreases to zero, construct $U_{j+1}=U_j+Z_j$ with independent $Z_j\sim\chi^2_2$ and the same $E_\perp$.
Then $E_\perp/(U_{j+1}+E_\perp)\le E_\perp/(U_j+E_\perp)$ almost surely, with a strict decrease in each nontrivial tail probability.
Moreover, Markov's inequality gives the explicit bound
\begin{equation}
 \ell_j\le\frac{d-1}{\varepsilon(d+2j)}\longrightarrow0.
 \label{eq:component-finite}
\end{equation}
Only the finite set $\mathcal J$ can impose a nontrivial constraint: if $\ell_j\le\alpha$, then $g_j(t)\le\alpha$ for every nonnegative threshold.
For $j\in\mathcal J$, the constraint $g_j(t)\le\alpha$ is exactly
$t\ge Q_{d+2j}(\alpha/\ell_j)$, because the chi-square survival function is continuous and strictly decreasing.
Adding $t\ge c_\alpha$ for null control proves both the validity and the componentwise minimality of \cref{eq:envelope}.

\subsection{Strict improvement over a confidence ball}
For every $j$, the event $\{S_j\ge t,L_j>\varepsilon\}$ implies
$E_\perp>\varepsilon S_j\ge\varepsilon t$.
Thus
\begin{equation}
 g_j(t)\le\Prob\{\chi^2_{d-1}>\varepsilon t\}.
\end{equation}
All component constraints are satisfied by $t=\chi^2_{d-1,1-\alpha}/\varepsilon$.
The extra null constraint therefore gives \cref{eq:strict-improvement}.
The strict final inequality follows from $c_\alpha<c_\alpha/\varepsilon$ and
$\chi^2_{d-1,1-\alpha}<\chi^2_{d,1-\alpha}$.
The latter follows by coupling a $\chi^2_d$ variable as an independent $\chi^2_{d-1}$ plus $\chi^2_1$, both with positive densities on their support.
Since the noncentral chi-square distribution has positive density at every positive finite argument, reducing the threshold strictly increases action probability at every finite $\kappa$.
This proves \cref{thm:envelope}.

The threshold is the smallest satisfying the componentwise sufficient constraints together with null control.
These constraints can be conservative for the original joint bound: the Poisson weights cannot place all their mass on an arbitrary positive $j$, so $\sup_j g_j(t)$ may exceed $\sup_{\kappa\ge0}\E_{\Pois(\kappa/2)}g_J(t)$, and $t_{\rm mix}$ is minimal for the componentwise scheme rather than for the joint bound itself.

\section{Certificates with an estimated noise scale}
\label{app:unknown}
\subsection{The estimated-scale confidence-ball rule}
Suppose $s>0$ is unknown and $N\ge2$.
The full contrast sample estimates it through
\begin{equation}
 \widehat s^2=\frac1\nu\sum_{i=1}^N\|D_i-\bar D\|^2,
 \qquad \nu=d(N-1),\qquad
 T=\frac{N\|\bar D\|^2}{d\widehat s^2}.
 \label{eq:variance}
\end{equation}
Pooling over coordinates is valid because the noise is isotropic and Gaussian.

\begin{theorem}[Confidence-ball action with estimated scale]
\label{thm:f-certificate}
For $\alpha,\varepsilon\in(0,1)$, let $q_\alpha=F_{d,\nu;1-\alpha}$ be the $1-\alpha$ quantile of the central $F$ distribution with degrees of freedom $d,\nu$; $F_{d,\nu}(\kappa)$ denotes its noncentral version.
The rule that acts with the mean-direction editor exactly when $T\ge q_\alpha/\varepsilon$ satisfies \cref{eq:joint} uniformly over $\delta\ne0$ and $s>0$, and takes action at the null with probability at most $\alpha$.
Its exact action probability is
\begin{equation}
 \Prob_{s,u}(A)=\SF_{F_{d,\nu}(\kappa)}(q_\alpha/\varepsilon).
 \label{eq:f-power}
\end{equation}
\end{theorem}

\subsection{Proof of \cref{thm:f-certificate}}
An orthogonal transformation of the $N$ samples in each coordinate separates their normalized mean from $N-1$ centered Gaussian coordinates.
Summing residual squares over $d$ coordinates gives
\begin{equation}
 W=\nu\widehat s^2/s^2\sim\chi^2_\nu,
 \qquad W\perp X,\qquad\nu=d(N-1).
\end{equation}
Hence
\begin{equation}
 \frac{N\|\bar D-\delta\|^2}{d\widehat s^2}
 =\frac{\|X-\sqrt\kappa u\|^2/d}{W/\nu}\sim F_{d,\nu}.
\end{equation}
The ball $B(\bar D,\sqrt{d\widehat s^2q_\alpha/N})$ has exact coverage $1-\alpha$ at every $\delta$ and $s$.
On $T\ge q_\alpha/\varepsilon$, it excludes zero and its tangent bound is at most $\varepsilon$.
The joint and null guarantees in \cref{thm:f-certificate} follow exactly as for known scale.
The noncentral law $T\sim F_{d,\nu}(\kappa)$ proves \cref{eq:f-power}.

\subsection{A finite-mixture improvement with estimated scale}
\begin{corollary}[A stronger certificate with estimated scale]
\label{cor:f-envelope}
With $\ell_j,\mathcal J$ from \cref{eq:component-tails}, let $Q_{m,\nu}^F(z)=F_{m,\nu;1-z}$ and
\begin{equation}
 a_{\rm mix}=\max\left\{q_\alpha,\ \max_{j\in\mathcal J}
 \frac{d+2j}{d}Q_{d+2j,\nu}^F(\alpha/\ell_j)\right\}.
 \label{eq:f-envelope}
\end{equation}
The rule $A=\{T\ge a_{\rm mix}\}$ has joint bad-action and null action probabilities at most $\alpha$, uniformly over unknown $s>0$.
It satisfies $a_{\rm mix}<q_\alpha/\varepsilon$ and therefore strictly increases action probability over the confidence-ball $F$ rule at every finite $\kappa\ge0$.
\end{corollary}
The proof conditions on the same Poisson component and uses the independent residual variance.

Let $W\sim\chi^2_\nu$ be independent of $(U_j,E_\perp)$.
At threshold $a$ for $T=(S_j/d)/(W/\nu)$, the conditional joint probability is
\begin{equation}
 g^F_j(a)=\ell_j\SF_{F_{d+2j,\nu}}\!\left(\frac{da}{d+2j}\right).
\end{equation}
The same finite component argument yields \cref{eq:f-envelope} and its joint guarantee; $a\ge q_\alpha$ gives null control.
For the domination bound, a bad action implies
$E_\perp>\varepsilon da W/\nu$.
Consequently all components are controlled if
\begin{equation}
 a\ge\frac{d-1}{d\varepsilon}F_{d-1,\nu;1-\alpha}.
\end{equation}
Thus
\begin{equation}
 a_{\rm mix}\le\max\left\{q_\alpha,
 \frac{d-1}{d\varepsilon}F_{d-1,\nu;1-\alpha}\right\}
 <q_\alpha/\varepsilon.
\end{equation}
For strictness, couple the two numerators with a shared independent denominator $W/\nu$.
Then $\chi^2_{d-1}/(W/\nu)<\chi^2_d/(W/\nu)$ almost surely under the additive coupling, implying the strict quantile inequality
$(d-1)F_{d-1,\nu;1-\alpha}<dF_{d,\nu;1-\alpha}$.
Because the noncentral $F$ density is positive, the lower threshold strictly increases the action probability.
The thresholds depend only on $d$, $\nu$, $\alpha$ and $\varepsilon$, so the rule needs neither $s$ nor $\kappa$.

The counterpart of \cref{eq:joint-moments} with $F$ quantiles follows by the same conditioning on $J$ and the independent denominator $W/\nu$.
For fixed $d,\kappa$ and $N\to\infty$ it recovers the known-scale rule and its action law.

\section{Proof of Corollary~\lref{cor:phase}}
\label{app:phase}
\begin{corollary}[A sharp constant-strength action boundary]
\label{cor:phase}
Suppose $d\to\infty$, $\kappa/d\to c>0$, and $p\ge p_0>0$.
Then $\risk_d^{\rm sel}(p,\kappa)\to1/(1+c)$ uniformly over such $p$.
For fixed $\alpha,\varepsilon\in(0,1)$, the known-scale and $F$ rules, both confidence-ball and finite-mixture versions, have action probabilities tending to one when $c>\varepsilon^{-1}-1$ and to zero when $c<\varepsilon^{-1}-1$.
The $F$ result permits any sequence $N\ge2$, including fixed $N$.
\end{corollary}

Rotate $u$ to $e_1$ and suppose $\kappa_d/d\to c>0$.
In the direct Gaussian representation,
\begin{equation}
 \frac{(\sqrt{\kappa_d}+Z_1)^2}{d}\longrightarrow c,
 \qquad \frac1d\sum_{i=2}^d Z_i^2\longrightarrow1
\end{equation}
in probability.
Hence $L_1\to(1+c)^{-1}$ in probability and, because it is bounded, in $L^1$.
For any event $A$ with probability at least $p_0>0$,
\begin{equation}
 \left|\E[L_1\mid A]-\frac1{1+c}\right|
 \le\frac{\E|L_1-(1+c)^{-1}|}{p_0}\longrightarrow0.
\end{equation}
Apply this to the exact selective attaining rules for all $p\ge p_0$ to prove the first part of \cref{cor:phase} uniformly over action frequencies.

For fixed $\alpha$, concentration gives $\chi^2_{d,1-\alpha}/d\to1$.
When $\nu=d(N-1)\ge d$, $W/\nu\to1$ and the central $F_{d,\nu}$ distribution concentrates at one, so $q_\alpha\to1$.
The observed normalized energies satisfy
$\|X\|^2/d\to1+c$ and $T\to1+c$.
The ball thresholds, normalized as $c_\alpha/(d\varepsilon)$ and $q_\alpha/\varepsilon$, both tend to $1/\varepsilon$.
This proves their action limits strictly above and below the boundary.

The mixture thresholds satisfy the same limits:
\begin{equation}
 t_{\rm mix}/d\longrightarrow1/\varepsilon,
 \qquad a_{\rm mix}\longrightarrow1/\varepsilon.
 \label{eq:envelope-limit}
\end{equation}
The domination bounds prove the upper limits.
For the lower limits, suppose along a subsequence one of these normalized thresholds stays at most $1/\varepsilon-\eta$ for some $\eta>0$.
Choose $c'\in(0,1/\varepsilon-1)$ with $1+c'>1/\varepsilon-\eta$ and use $\kappa_d=c'd$.
Under this sequence the proposed action probability tends to one, but $L_1\to1/(1+c')>\varepsilon$, so joint bad action tends to one.
This contradicts the finite-dimensional uniform bound $\alpha<1$.
Thus \cref{eq:envelope-limit} holds, and the same comparison of statistic and threshold proves the mixture action limits.
At $c=1/\varepsilon-1$ the ratio limit alone does not determine the action limit: second-order deviations of $\kappa_d$ and quantile fluctuations can affect it.

\section{Proof of Proposition~\lref{prop:anisotropy}}
\label{app:anisotropy}
\begin{proof}
(i) Let $w=\bar D/\|\bar D\|$, with $(u^\top w)^2$ set to zero when $\bar D=0$.
By the Cauchy--Schwarz inequality, $(\E|u^\top\bar D|)^2\le\E[(u^\top w)^2]\,\E\|\bar D\|^2$.
Since $\E|u^\top\bar D|\ge|\E u^\top\bar D|=\theta$ and $\E\|\bar D\|^2=\theta^2+\tr\Gamma/N$, we get $\E(u^\top w)^2\ge\bar\kappa/(d+\bar\kappa)$.
Given $w$, uniform completion in $w^\perp$ multiplies the residual $1-(u^\top w)^2$ by $(d-k)/(d-1)$ in expectation, as in \cref{app:frontier}.
Only the mean and covariance of $\bar D$ were used.

(ii) By \cref{eq:kernel}, it suffices to consider rules that choose a deleted subspace $K$ of dimension at most $k$.
Fix $m>k$ and let $S_m$ be spanned by eigenvectors of the $m$ largest eigenvalues of $\Gamma$; the maximum risk over $\sphere$ is at least that over unit $u\in S_m$.
For such $u$, the coordinates of the contrasts outside $S_m$ are Gaussian, independent of those inside, and have a law free of $u$, so they act only as fitting randomness.
For $u\in S_m$ the captured fraction is $u^\top\Pi_Ku=u^\top Au$ with $A=\Pi_{S_m}\Pi_K\Pi_{S_m}$, which satisfies $0\preceq A\preceq\Pi_{S_m}$ and $\tr A\le k$.
Such an $A$ is a convex combination of orthogonal projectors of rank at most $k$ inside $S_m$, each dominated by a rank-$k$ projector inside $S_m$, so a randomized rank-$k$ deletion inside $S_m$ has no larger loss at every $u\in S_m$.
Inside $S_m$ the noise covariance of $\bar D$ dominates $\gamma_{(m)}\Id/N$, so the observation can be simulated from one with noise covariance exactly $\gamma_{(m)}\Id/N$ by adding independent Gaussian noise; the restricted problem is therefore no easier than the isotropic $m$-dimensional problem, whose minimax risk is $\risk_m^*(k,N\theta^2/\gamma_{(m)})$ by \cref{thm:frontier}.
Revealing $\Gamma$ cannot increase the minimax risk, so the bound also holds when $\Gamma$ is unknown.
\end{proof}
For steep spectra (ii) can be loose: for the participant-level noise spectrum of the BOAS experiment (\cref{app:boas}) it gives $0.30$ at $P=n=8$, against $0.52$ from (i).
A noise-weighted direction $(\Id+gC)^{-1}\bar D$ does not close this gap.
Even with $g$ tuned on the test participants of splits other than those in \cref{tab:boas-grid}, it lowers the erasers' residual by at most $0.065$ over the grid and by $0.005$ at $P=n=8$.

\section{Simulation protocol and checks}
\label{app:numerics}
\begin{table}[htbp]
 \centering
 \small
 \setlength{\tabcolsep}{4pt}
 \caption{\textbf{Exact predictions against Monte Carlo estimates.}
 Means carry $\pm1.96$ Monte Carlo SE ($8{,}192$ draws for the frontier, $512$ fits per eraser cell, $256$ calibrations for INLP), and probabilities carry exact $95\%$ binomial intervals.
 Real-content rows list the three pools at $d=64$, each with one value shared by MP, SAL and LEACE.
 Device rows give the share of the population mean shift left in held-out participants, as mean $\pm1.96$ SE over random splits of the same participants ($20$ for BOAS, $40$ for Flex), and its prediction from the calibration pool (\cref{app:boas,app:real-pairs}).}
 \label{tab:checks}
 \begin{tabular}{llcc}
  \toprule
  Quantity & Setting & Predicted & Observed \\
  \midrule
  Cost-one residual & $d=256$, $\kappa=128$ & $0.6652$ & $0.6644\pm0.0009$ \\
  Cost-one residual & $d=256$, $\kappa=512$ & $0.3326$ & $0.3325\pm0.0006$ \\
  Eraser residual, Gaussian content & $d=N=64$ & $0.4961$ & $0.4922\pm0.0067$ \\
  Eraser residual, real content & $N_\star=63$, target $0.50$ & $0.5000$ & $0.504,\ 0.501,\ 0.503$ ($\pm0.007$) \\
  Eraser residual, real content & $N_\star=190$, target $0.25$ & $0.2495$ & $0.251,\ 0.247,\ 0.248$ ($\pm0.004$) \\
  INLP excess over MP & $d=N=32$, $k=1$ & $0.181$ & $0.186\pm0.012$ \\
  INLP excess over MP & $d=N=32$, $k=4$ & $0.012$ & $0.017\pm0.008$ \\
  Action prob., confidence ball & $d=\kappa=256$ & $0.0289$ & $0.026\ [0.023,0.030]$ \\
  Action prob., finite mixture & $d=\kappa=256$ & $0.1152$ & $0.114\ [0.107,0.121]$ \\
  \midrule
  Wearable shift (BOAS), one participant & $d=118$, $P=1$, $n=128$ & $0.757$ & $0.690\pm0.033$ \\
  Wearable shift (BOAS) & $d=118$, $P=8$, $n=8$ & $0.501$ & $0.507\pm0.024$ \\
  Wearable shift (BOAS) & $d=118$, $P=32$, $n=128$ & $0.101$ & $0.125\pm0.027$ \\
  Detection (BOAS) & $d=118$, $P=8$, $n=8$ & $0.95$ & $0.97$ \\
  Device shift (Flex) & $d=225$, $P=8$, $n=16$ & $0.508$ & $0.504\pm0.030$ \\
  \bottomrule
 \end{tabular}
\end{table}

The studies below simulate the Gaussian experiment with $\alpha=0.05$ and $\varepsilon=0.5$.
Seeds, grids and code are in the code repository (\coderepo).
All analyses ran on the CPU of a single laptop (Apple M3 Pro, 36 GB RAM), except the one-time extraction of the frozen EEGPT features, which used its integrated GPU.
Excluding the downloads and feature extraction of the EEG recordings, the complete set of analyses takes well under an hour.
Probability error bars are exact $95\%$ binomial intervals, mean error bars are $\pm1.96$ Monte Carlo standard errors, and rules being compared share the same draws.
A displayed zero means that no event occurred among the draws, as for the confidence-ball rule at $d=256$, $\kappa=128$ (\cref{app:certificates}).

\subsection{Hard-budget and selective frontiers}
\begin{figure}[htbp]
 \centering
 \includegraphics[width=\linewidth]{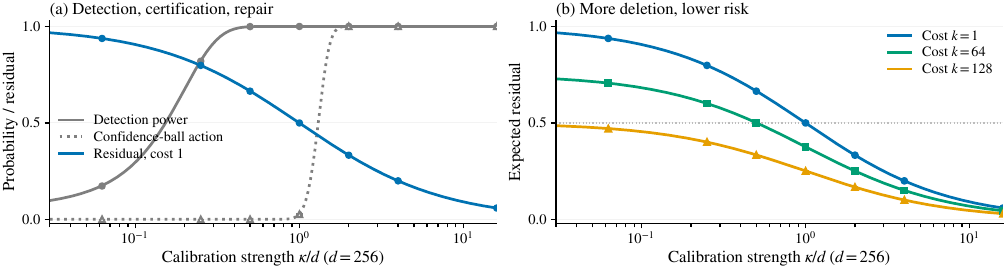}
 \caption{\textbf{Detection, deletion, and certified action are different quantities.}
 At $d=256$, lines are exact formulas and points use $8{,}192$ draws per strength.
 (a) Detection power, confidence-ball action probability, and the optimal cost-one expected residual.
 (b) The hard-budget frontier for three deletion dimensions.}
 \label{fig:frontier}
\end{figure}
The frontier grid uses $d\in\{16,64,256\}$, $\kappa/d\in\{0,\tfrac1{16},\tfrac14,\tfrac12,1,2,4,16\}$, deletion ranks $k\in\{1,d/4,d/2,d\}$ and $8{,}192$ draws per cell.
For $1<k<d$ the sampler multiplies the rank-one residual by the beta-distributed fraction that random orthogonal completion leaves (\cref{app:frontier}).

The selective curves use $d=64$, $\kappa\in\{8,32,128\}$ and action frequencies from $10^{-10}$ to $1$.
They are checked against the Bessel-ratio form of the posterior residual,
\begin{equation}
 h_\kappa(r)=\frac{d-1}{\sqrt\kappa r}
 \frac{I_{d/2}(\sqrt\kappa r)}{I_{d/2-1}(\sqrt\kappa r)},
 \label{eq:bessel}
\end{equation}
integrated against the noncentral radial density on the accepted tail, where $I_\eta$ is the modified Bessel function of the first kind.

\begin{figure}[htbp]
 \centering
 \includegraphics[width=\linewidth]{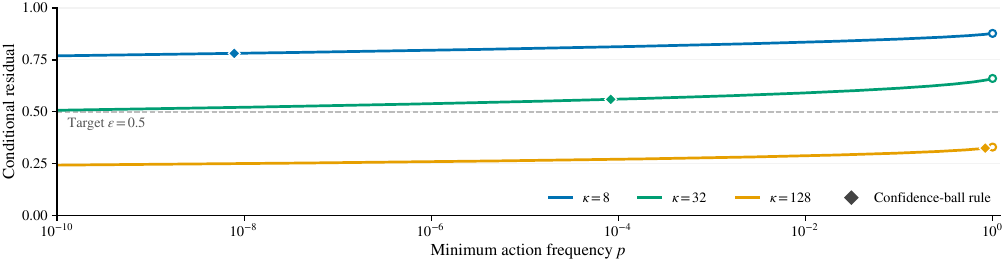}
 \caption{\textbf{Conditional improvement has an action-frequency cost.}
 Exact selective frontiers at $d=64$.
 Open circles mark the ungated $p=1$ limits, and diamonds the confidence-ball rule's action probability and conditional mean residual.
 The dashed line is $\varepsilon=0.5$.
 A diamond may lie above it, because the certificate controls joint bad action.
 The oracle curves fix $\kappa$, whereas the confidence-ball rule does not know it.}
 \label{fig:selective}
\end{figure}

\subsection{Optimal allocation under an average budget}
At $d=16$ and $\kappa=8$, $131{,}072$ shared calibration draws cover budgets from $0$ to $16$, and a direct numerical optimization over rules reproduces \cref{eq:average-frontier}.
At $b=4$ the average-budget risk is $0.499$ (simulated $0.499\pm0.001$ at mean cost $3.99\pm0.02$), against $0.514$ under the hard budget: the optimal rule deletes one direction on four fifths of the fits and all sixteen on the rest (\cref{fig:average}).

\begin{figure}[htbp]
 \centering
 \includegraphics[width=\linewidth]{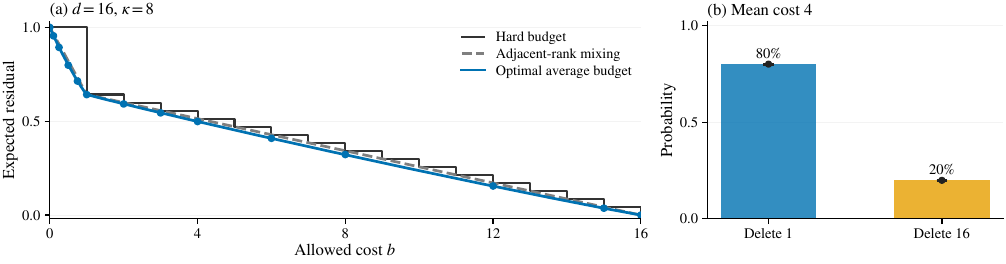}
 \caption{\textbf{An average budget rewards adaptive deletion rank.}
 Left: exact hard-budget, adjacent-rank mixing, and optimal average-budget risks at $d=16$, $\kappa=8$, with points from $131{,}072$ shared draws.
 Right: at average cost four the optimal rule deletes one direction with probability $0.8$ and all sixteen with probability $0.2$.
 Black points with exact $95\%$ binomial intervals check these probabilities.}
 \label{fig:average}
\end{figure}

\subsection{Action certificates with known and estimated scale}
The known-scale rules are evaluated at $d\in\{4,16,64,256\}$.
The estimated-scale rules use a grid of $d\in\{4,16,64\}$, $N\in\{2,4,16\}$ and $\kappa/d\in\{0,\tfrac12,1,2,4\}$, with $4{,}096$ draws per cell.
A numerical check confirms that components beyond the cutoff $\mathcal J$ never constrain the mixture threshold.
The null action rate stays at or below $\alpha$ up to Monte Carlo error, and on the same draws the mixture rule never acts less often than the confidence-ball rule (\cref{fig:gate-comparison}).

\begin{figure}[htbp]
 \centering
 \includegraphics[width=\linewidth]{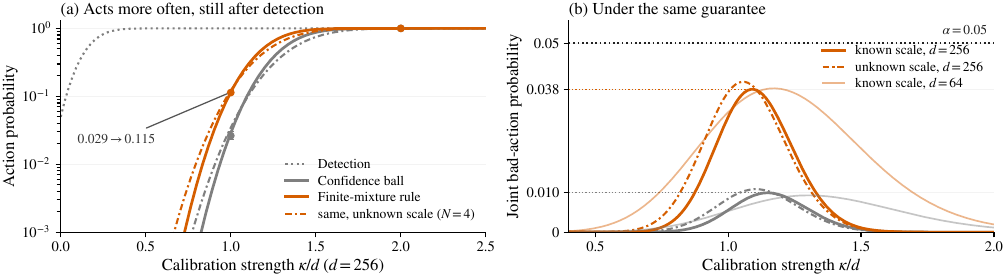}
 \caption{\textbf{More frequent action under the same uniform joint-error guarantee.}
 At $d=256$: (a) exact action probabilities of the confidence-ball and finite-mixture rules with known scale (solid) and with the noise scale estimated from $N=4$ pairs (dash-dotted), Monte Carlo estimates on shared draws (points), and the power of the radial test (dotted).
 (b) The exact joint probability of acting and missing the target, which every rule keeps below $\alpha$ at all strengths.
 Thin lines repeat the known-scale rules at $d=64$.}
 \label{fig:gate-comparison}
\end{figure}

\section{Fitted erasers on Gaussian and real content}
\label{app:eraser-illustrations}
These experiments apply \cref{cor:empirical-erasers,cor:real-content-frontier} to fitted MP, SAL and LEACE maps, first with Gaussian content, then with real represented content.

\subsection{Gaussian content}
For $d\in\{16,64\}$ and $N/d\in\{\tfrac14,\tfrac12,1,2,4\}$ we fit MP, centered rank-one SAL and binary LEACE (\cref{app:empirical-erasers}) to $512$ calibration samples of the paired construction of \cref{app:experiment} with $\theta=s^2=1$.
The three kernels coincide on every fit, and the mean residuals follow \cref{eq:empirical-eraser-risk} (\cref{tab:checks}).
The methods differ only in distortion: at $d=N=64$, LEACE's median linear cost is $2.13$ against one for MP and SAL, and it grows sharply near the pooled-covariance rank transition $2N-1=d$.

\subsection{Real content pools and the semisynthetic check}
\label{app:real-content-pools}
Let $z_1,\ldots,z_m\in\R^p$ be a frozen real representation pool, with
\begin{equation}
 \bar z=\frac1m\sum_{j=1}^m z_j,
 \qquad
 \widehat\Sigma_m=\frac1m\sum_{j=1}^m(z_j-\bar z)(z_j-\bar z)^\top.
\end{equation}
Choose a $d$-dimensional positive-eigenvalue subspace, with eigenvectors $U_d$ and eigenvalues $\Lambda_d$, set $y_j=\Lambda_d^{-1/2}U_d^\top(z_j-\bar z)$, and draw
\begin{equation}
 C=\sqrt{1-s^2/2}\,y_T,
 \qquad T\sim\operatorname{Unif}\{1,\ldots,m\}.
\end{equation}
Conditional on the pool, this law satisfies \cref{eq:real-content-moments} exactly, and independent pool indices and endpoint noises give independent calibration pairs.
The content comes from real recordings, while the shift and the measurement noise are synthetic.

\parhead{Pools.}
\label{app:real-content-experiment}
The three pools come from the paired Neuroscan/Flex recordings of \citet{williams2020flex} (\cref{app:real-pairs}).
Two of them describe the same N170 events, once by temporal features and once by frozen EEGPT features.
The third holds $5{,}330$ MMN events from $15$ of the $16$ participants.
Each content vector is the midpoint of the paired device representations, reduced to its first $16$ or $64$ principal coordinates and whitened over the pool.
Whitening keeps the non-Gaussian structure: the radial fourth moment of the MMN features is $1.27$ times its Gaussian value at $d=16$ and $1.19$ times at $d=64$.

\parhead{Design.}
We set $\theta=1$ and cross the three pools with $d\in\{16,64\}$, $s^2\in\{0.25,1,1.75\}$ and $N/d\in\{\tfrac14,\tfrac12,1,2,4\}$.
At $s^2=1$ and for each target $\varepsilon\in\{\tfrac12,\tfrac14\}$ we also use the calibration sizes $\lceil N_\star/2\rceil$, $N_\star$ and $2N_\star$, where
\begin{equation}
 N_\star(\varepsilon;d,\theta,s)
 :=\min\left\{N\in\mathbb N,\ N\ge1:\risk_d^*(1,N\theta^2/s^2)\le\varepsilon\right\}
 \label{eq:eraser-calibration-budget}
\end{equation}
is the smallest calibration size with expected residual at most $\varepsilon$.
It is well defined because the risk decreases in $\kappa$.
Each of the $120$ cells has $512$ calibration samples.

\parhead{Results.}
The loss is evaluated from the row space of each fitted matrix under the exact population covariance of the construction, so it carries no test-sample noise (\cref{fig:real-content-risk}).
Cell means deviate from the predictions by $0.004$ in root mean square and at most $0.015$.
At the calibration sizes $N_\star$ all three pools land on their targets: $0.501$--$0.504$ for the target $\tfrac12$ at $N=63$, and $0.247$--$0.251$ for the target $\tfrac14$ at $N=190$.
Because the shared content cancels from each paired contrast, the three pools are indistinguishable, as \cref{cor:real-content-frontier} predicts.

\begin{figure}[htbp]
 \centering
 \includegraphics[width=\linewidth]{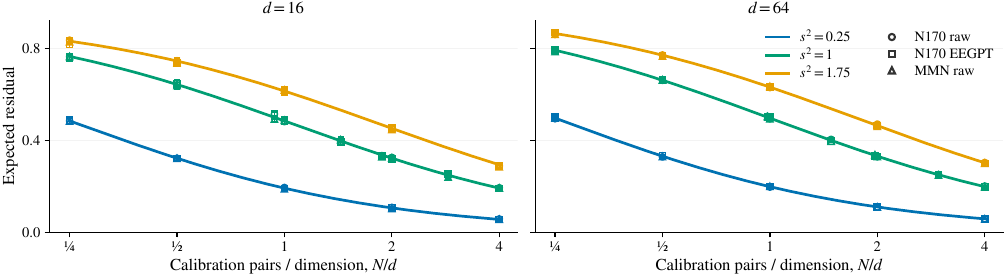}
 \caption{\textbf{An illustration of the non-Gaussian-content extension.}
 Lines are exact predictions, and points are fitted-map population residuals averaged over $512$ calibration samples ($\pm1.96$ Monte Carlo SE).
 MP, SAL and LEACE share the residual at every fit, so each estimate is shown once.
 Markers distinguish the three pools of real represented content.
 Shift and noise are synthetic.}
 \label{fig:real-content-risk}
\end{figure}

\section{Wearable versus clinical EEG (BOAS)}
\label{app:boas}
This appendix gives the protocol and results of \cref{sec:real-device} and repeats the analysis on a second device pair.
Preprocessing details and code are in the code repository.

\parhead{Data and features.}
The Bitbrain Open Access Sleep dataset \citep[OpenNeuro ds005555, v1.1.3, CC0;][]{lopezlarraz2026boas} contains $128$ nights of $100$ healthy adults recorded simultaneously by a clinical polysomnograph and a forehead EEG headband, with consensus expert staging in $30$-s epochs \citep{esparzaiaizzo2026sleep}.
The study was approved by the ethics committee of Arag\'on (C.I.\ PI24/046), and no new data were recorded.
We take one night per participant, drawn with a fixed seed before any data were downloaded.
The two headband channels are paired with the nearest polysomnograph sites on the same side (F3, F4).
Each artifact-free N2 epoch is described by log Welch power in $0.5$ Hz bins over $0.5$--$30$ Hz, so $d=118$ (other bin widths give $d=28$, $58$ and $236$).
The contrast is the polysomnograph feature minus the headband feature of the same epoch.
A night is excluded when, for either channel pair, the median correlation of the paired $0.5$--$4$ Hz signals over $5$-min windows falls below $0.5$, which flags a defective headband channel in $14$ of the $100$ nights.
One further participant is excluded for having fewer than $128$ N2 epochs.
The remaining $85$ participants contribute $325$ to $752$ N2 epochs each.

\parhead{Design.}
Each of $20$ random splits divides the participants into a calibration pool of $43$ and a test set of $42$.
Coordinates are whitened by the pool's within-source covariance, estimated from about $24{,}000$ epochs per device (some $200$ per dimension at $d=118$).
From the pool alone we estimate the shift $\delta$ and, from off-diagonal products of participant means as in \cref{eq:real-pair-statistic} below, its squared length.
The within-participant covariance $\Sigma_w$ is averaged over participants.
The between-participant covariance $\Sigma_b$ is the covariance of participant means minus the average of $\Sigma_w/n_j$, with negative eigenvalues set to zero.
A calibration draw takes $P\in\{1,2,3,4,6,8,12,16,24,32\}$ pool participants with $n\in\{1,2,4,8,16,32,64,128\}$ random epochs each and averages their means into $\bar D$.
MP, SAL and LEACE fitted to these draws share the same kernel.
The prediction is the erasers' residual under $\bar D\sim N(\delta,(\Sigma_b+\Sigma_w/n)/P)$, and the closed-form bound $d/(d+\bar\kappa)$ is at most $0.031$ above it.
Detection uses the radial statistic with its null law under the same estimated covariance.

\parhead{Statistic.}
The observed residual uses only the test participants.
Let $m_j$ be the mean contrast of test participant $j$ in the whitened coordinates, with $G$ test participants.
For a fitted map $M$, let $P_M=M^\top(MM^\top)^+M$ be its row-space projector and define
\begin{equation}
 \widehat Q(M)=\frac{1}{G(G-1)}\sum_{j\ne k}m_j^\top P_Mm_k,
 \qquad
 \widehat R(M)=\frac{\widehat Q(M)}{\widehat Q(\Id)}.
 \label{eq:real-pair-statistic}
\end{equation}
Products of different participants omit the noise of each participant's own mean, so for independent participants $\widehat Q(M)$ estimates the squared projected population mean contrast without bias.
The ratio is not clipped, so test noise can take it outside $[0,1]$.
For the subtraction of $\bar D$, the same products of $m_j-\bar D$ give the remaining separation.

\parhead{Results.}
\Cref{tab:boas-grid} lists a subset of the grid and \cref{tab:boas-robust} the agreement under other choices.
Here $\theta^2\approx2.6$, $\tr\Sigma_b\approx9.0$ and $\tr\Sigma_w\approx100$, so one participant caps $\bar\kappa$ at about $35$.
Since $\tr\Sigma_w/\tr\Sigma_b\approx11$, \cref{cor:two-level} gives $n^\star=\sqrt{11c}$ epochs per participant when a participant costs as much as $c$ epochs.
The largest deviations are for one participant with many epochs, where the prediction is pessimistic ($0.76$ against $0.69$), and for $32$ participants with $128$ epochs, where it is optimistic ($0.10$ against $0.125\pm0.027$).
With a single participant or four epochs or fewer in total, the radial test detects more often than predicted ($0.16$ to $0.27$ against $0.07$ to $0.16$), consistent with heavier-than-Gaussian epochs.
At $P=n=8$, the cell of \cref{sec:real-device}, observed and predicted detection are $0.97$ and $0.95$.
Contiguous blocks of epochs leave more shift than random epochs because neighbouring epochs are correlated ($0.56$ against $0.51$ at $P=n=8$), and the covariance of block means accounts for most of this (\cref{tab:boas-robust}).

\begin{table}[htbp]
 \centering
 \small
 \setlength{\tabcolsep}{4pt}
 \caption{BOAS, N2 epochs, $d=118$, $85$ participants. The eraser columns give the share of the population mean shift left in held-out participants after MP, SAL and LEACE, observed (mean $\pm1.96$ SE over $20$ random splits of the same participants) and predicted from the calibration pool by the plug-in two-level prediction and by the frontier formula at $\bar\kappa$. Detection is the rejection probability of the level-$0.05$ radial test. The last two columns give the separation left by subtracting $\bar D$, relative to no correction.}
 \label{tab:boas-grid}
 \begin{tabular}{rrccccccc}
  \toprule
  & & \multicolumn{3}{c}{Erasers} & \multicolumn{2}{c}{Detection} & \multicolumn{2}{c}{Subtract $\bar D$} \\
  $P$ & $n$ & Observed & Plug-in & At $\bar\kappa$ & Obs. & Pred. & Obs. & Pred. \\
  \midrule
  1 & 1 & $0.960\pm0.006$ & $0.967$ & $0.969$ & $0.20$ & $0.07$ & $48.28$ & $42.27$ \\
  1 & 8 & $0.858\pm0.013$ & $0.874$ & $0.885$ & $0.23$ & $0.12$ & $8.97$ & $8.33$ \\
  1 & 128 & $0.690\pm0.033$ & $0.757$ & $0.783$ & $0.23$ & $0.13$ & $3.94$ & $3.78$ \\
  4 & 1 & $0.900\pm0.010$ & $0.904$ & $0.905$ & $0.27$ & $0.16$ & $11.57$ & $10.57$ \\
  4 & 8 & $0.653\pm0.016$ & $0.663$ & $0.669$ & $0.51$ & $0.57$ & $2.19$ & $2.08$ \\
  4 & 128 & $0.501\pm0.030$ & $0.464$ & $0.481$ & $0.66$ & $0.66$ & $1.09$ & $0.94$ \\
  8 & 1 & $0.835\pm0.014$ & $0.833$ & $0.833$ & $0.36$ & $0.35$ & $5.67$ & $5.28$ \\
  8 & 8 & $0.507\pm0.024$ & $0.501$ & $0.505$ & $0.97$ & $0.95$ & $1.15$ & $1.04$ \\
  8 & 128 & $0.309\pm0.026$ & $0.307$ & $0.318$ & $0.99$ & $0.99$ & $0.50$ & $0.47$ \\
  32 & 1 & $0.566\pm0.018$ & $0.563$ & $0.564$ & $0.98$ & $0.99$ & $1.40$ & $1.32$ \\
  32 & 8 & $0.216\pm0.021$ & $0.203$ & $0.205$ & $1.00$ & $1.00$ & $0.29$ & $0.26$ \\
  32 & 128 & $0.125\pm0.027$ & $0.101$ & $0.105$ & $1.00$ & $1.00$ & $0.16$ & $0.12$ \\
  \bottomrule
 \end{tabular}
\end{table}

\begin{table}[htbp]
 \centering
 \small
 \caption{BOAS: agreement between observed and predicted eraser residual over the $80$ cells of the grid under other analysis choices. The last column is the one-participant cap $\theta^2d/\tr\Sigma_b$.}
 \label{tab:boas-robust}
 \begin{tabular}{lrrcccr}
  \toprule
  Setting & Participants & $d$ & RMS deviation & Max.\ deviation & Within $0.03$ & Cap \\
  \midrule
  N2, $d=118$ (main) & 85 & 118 & $0.016$ & $0.067$ & $95\%$ & $35$ \\
  N2, no channel rule & 97 & 118 & $0.021$ & $0.064$ & $88\%$ & $28$ \\
  All stages & 86 & 118 & $0.014$ & $0.058$ & $94\%$ & $36$ \\
  N2, $d=28$ & 85 & 28 & $0.023$ & $0.074$ & $78\%$ & $9$ \\
  N2, $d=58$ & 85 & 58 & $0.018$ & $0.053$ & $89\%$ & $17$ \\
  N2, $d=236$ & 85 & 236 & $0.016$ & $0.047$ & $96\%$ & $67$ \\
  N2, contiguous blocks & 85 & 118 & $0.033$ & $0.072$ & $64\%$ & $34$ \\
  \bottomrule
 \end{tabular}
\end{table}

\parhead{How much of the device difference is a mean shift.}
Linear erasers act only on the mean difference, so we measure how much of the divergence between the two devices' features is a mean shift.
The Jeffreys divergence between Gaussian fits splits into a mean part and a covariance part.
From each part we subtract its value under random within-pair swaps of the devices.
Within a single participant's night, the mean part accounts for a median of $0.76$ of the total (10th--90th percentile $0.36$--$0.91$), so there the devices differ mostly by a shift.
Pooled over the $85$ participants, its share is $0.31$.
This is consistent with shifts that vary across participants, which the two-level law captures through $\Sigma_b$.

\parhead{Planning from a pilot.}
The two-level law needs only $\theta^2$, $\tr\Sigma_b$ and $\tr\Sigma_w$.
To check whether a small pilot gives them, each of $20$ splits divides the participants into three disjoint sets.
The recordings of $20$ participants give the whitening, with each device used separately and without pairing.
Another $35$ form the calibration set, and the remaining $30$ are the test set.
A pilot consists of $P_0\in\{3,4,6,8\}$ calibration participants with their full N2 nights.
The three scalars are estimated from the pilot alone, and the prediction is the frontier formula at $\bar\kappa$.
With eight pilot nights, the one-participant cap has a median of $29$ against $28$ from all $35$ calibration participants, and it stays below $d=118$ in every pilot.
At $c=100$ the pilots give a median $n^\star$ of $31$ epochs.
Their predicted residual deviates from the observed one by a median of $0.10$ in root mean square ($0.05$--$0.20$), against $0.08$ with the scalars of all $35$ calibration participants, so most of the error comes from the finite test set.
The number of participants that leaves half the shift at $n=8$ is predicted as $9$ ($6$--$23$), against about $10$ observed.
Pilots of three or four nights still put the cap below $d$ in at least $95\%$ of cases.

\subsection{A second device pair: Neuroscan/Flex recordings}
\label{app:real-pairs}
Each of the $2{,}291$ simultaneously recorded Neuroscan/Flex N170 pairs from the $16$ participants of \citet{williams2020flex} is described by fifteen $50$-ms averages of $16$ common-average-referenced channels.
Removing the dependency introduced by the common reference leaves $d=225$ coordinates (details in the code repository).
Participants contribute $24$ to $250$ pairs.
The recordings are public on the Open Science Framework (\url{https://osf.io/zj3f5/}).
The original study was approved by the Macquarie University Human Research Ethics Committee (Ref.\ 5201831203493), and all participants gave written informed consent.

\parhead{Design.}
Each of $40$ random splits divides the participants into a calibration pool and a test set of $8$ each.
As for BOAS, the coordinates are whitened by the pool, the two-level prediction uses the pool's estimates of $\theta^2$, $\Sigma_b$ and $\Sigma_w$, and the observed residual is \cref{eq:real-pair-statistic} on the test participants.
We draw $P\in\{1,2,3,4,6,8\}$ calibration participants with $n\in\{1,2,4,8,16\}$ pairs each, ten times per split.
The grid stops at $n=16$ because the smallest participant has $24$ pairs.

\parhead{Result.}
Over the $30$ cells, the prediction deviates from the observed residual by $0.005$ in root mean square and at most $0.011$.
One participant caps the calibration strength at about $99$ in $d=225$, and eight participants with $16$ pairs each leave $0.504\pm0.030$ of the shift against the predicted $0.508$.
Subtracting $\bar D$ leaves more separation than no correction in every cell.

\section{INLP and R-LACE at fixed rank}
\label{app:subspace-experiment}
This experiment illustrates \cref{cor:subspace-excess} with fitted kernels that need not contain the empirical contrast.
The number of deletions is fixed in advance.
If a method chose its rank $K$ from the data, \cref{eq:subspace-conditional-risk} would still hold at the realized rank and give the lower bound $\risk_d^*(k_{\max},\kappa)$ whenever $K\le k_{\max}$, but not the fixed-rank identity.

\parhead{Design.}
Each replicate draws $U$ uniformly on $\sphere$ and $N$ pairs $H_i^\pm=C_i\pm U/2+E_i^\pm$ with $C_i,E_i^\pm\sim N(0,\Id/2)$, so that $\theta=s^2=1$, each source has covariance $\Id$, and $\kappa=N$.
The grid has $d\in\{16,32\}$, $N/d\in\{\tfrac12,1,2,4\}$ and $k\in\{1,4\}$, with $256$ calibrations per cell.
Every method receives the same $2N$ labelled endpoints.
MP, SAL and LEACE are evaluated at $k=1$, and the frontier reference deletes $w$ and $k-1$ uniform directions orthogonal to it.
Fixed-step INLP fits an $\ell_2$-regularized logistic classifier ($C=1$) to the successively projected endpoints and removes its coefficient direction, keeping the first and fourth iterates.
Fixed-budget R-LACE follows the authors' alternating updates, Fantope projection and final rounding \citep{ravfogel2022rlace} for $2{,}000$ full-batch iterations and returns an orthogonal projection with exactly $k$ deletions.

\parhead{Evaluation and prediction.}
The loss $L=\|P_MU\|^2$ is computed from the row space of each returned matrix, and $\rho_K=\|P_Mw\|^2$ from the empirical direction.
With $a=\sqrt\kappa\|X\|$, the Bessel-ratio form \cref{eq:bessel} gives
\[
 \lambda_\kappa(\|X\|)=\frac{I_{d/2}(a)}{a I_{d/2-1}(a)},\qquad
 g_\kappa(\|X\|)=1-d\lambda_\kappa(\|X\|),
\]
and the conditional prediction is $(d-k)\lambda+g\rho_K$.
To check the excess on the same calibrations, let $L_0$ be the loss of MP with random completion.
Its conditional mean is $(d-k)\lambda$, so $\overline\E[L-L_0]=\overline\E[g\rho_K]$.
We compare the observed mean of $L-L_0$ with the predicted mean of $g\rho_K$ through the standard error of their paired difference.

\parhead{Results.}
\Cref{tab:subspace-excess-check} reports every cell.
Over the $32$ INLP and R-LACE cells, observed and predicted excess differ by $0.004$ in root mean square and by at most $2.2$ paired standard errors.
The fitted R-LACE subspaces nearly contain $w$, whereas one-step INLP leaves a visible excess that its fitted subspace explains.
Four INLP steps shrink this excess at four times the cost of one step.
At $d=N=32$, LEACE's median linear cost is $2.09$ against one for the orthogonal maps, which is the arrow in \cref{fig:benchmark}a.

\begin{table}[htbp]
 \centering
 \small
 \setlength{\tabcolsep}{4pt}
 \caption{Complete controlled grid. The frontier is $\risk_d^*(k,N)$. Observed excess is the paired mean residual above MP plus random completion, with its Monte Carlo standard error over $256$ calibrations. Predicted excess is the mean of $g\rho_K$, which is below $2\times10^{-8}$ for R-LACE in every row.}
 \label{tab:subspace-excess-check}
 \begin{tabular}{rrrcrrr}
  \toprule
  $d$ & $N$ & $k$ & Frontier & \multicolumn{2}{c}{INLP excess} & R-LACE excess \\
   & & & & Observed $\pm$ SE & Predicted & Observed $\pm$ SE \\
  \midrule
  16 & 8 & 1 & $0.6423$ & $0.0683\pm0.0070$ & $0.0737$ & $0.0000\pm0.0000$ \\
  16 & 8 & 4 & $0.5138$ & $0.0036\pm0.0082$ & $0.0030$ & $0.0053\pm0.0079$ \\
  16 & 16 & 1 & $0.4838$ & $0.1236\pm0.0080$ & $0.1399$ & $0.0000\pm0.0000$ \\
  16 & 16 & 4 & $0.3871$ & $0.0083\pm0.0064$ & $0.0049$ & $-0.0088\pm0.0067$ \\
  16 & 32 & 1 & $0.3216$ & $0.1484\pm0.0071$ & $0.1471$ & $0.0000\pm0.0000$ \\
  16 & 32 & 4 & $0.2572$ & $0.0042\pm0.0043$ & $0.0013$ & $0.0028\pm0.0039$ \\
  16 & 64 & 1 & $0.1914$ & $0.1128\pm0.0049$ & $0.1078$ & $0.0000\pm0.0000$ \\
  16 & 64 & 4 & $0.1531$ & $-0.0017\pm0.0022$ & $0.0002$ & $-0.0008\pm0.0025$ \\
  32 & 16 & 1 & $0.6548$ & $0.1103\pm0.0056$ & $0.1071$ & $0.0000\pm0.0000$ \\
  32 & 16 & 4 & $0.5914$ & $0.0159\pm0.0047$ & $0.0124$ & $0.0070\pm0.0039$ \\
  32 & 32 & 1 & $0.4921$ & $0.1863\pm0.0060$ & $0.1807$ & $0.0000\pm0.0000$ \\
  32 & 32 & 4 & $0.4444$ & $0.0171\pm0.0038$ & $0.0122$ & $0.0023\pm0.0034$ \\
  32 & 64 & 1 & $0.3275$ & $0.1722\pm0.0051$ & $0.1673$ & $0.0000\pm0.0000$ \\
  32 & 64 & 4 & $0.2958$ & $0.0041\pm0.0023$ & $0.0025$ & $0.0028\pm0.0022$ \\
  32 & 128 & 1 & $0.1957$ & $0.1180\pm0.0034$ & $0.1149$ & $0.0000\pm0.0000$ \\
  32 & 128 & 4 & $0.1768$ & $-0.0001\pm0.0014$ & $0.0003$ & $-0.0012\pm0.0014$ \\
  \bottomrule
 \end{tabular}
\end{table}

\section{Anisotropic content: a numerical check}
\label{app:anisotropy-check}
The main results assume isotropic within-source variation and calibration noise, and this appendix checks the residual of the empirical-direction editor under anisotropic content.

The content covariance has eigenvalues proportional to $1/j$ or $1/j^2$ at $d=64$, the noise level is $s^2\in\{0.1,1\}$, and the shift points along a random direction or along the top or bottom content eigenvector.
After whitening, the contrast noise has effective dimension $49$ to $60$.
Over $30$ cells with $4{,}000$ calibrations each, the frontier \cref{eq:frontier} at the whitened $\kappa=N\|\tilde\delta\|^2/\bar\sigma^2$, with $\bar\sigma^2$ the mean whitened noise variance, matches the residual of the empirical-direction editor to $0.006$ in root mean square and $0.015$ at most: anisotropy enters mainly through $\kappa$.
The check concerns the empirical-direction editor and therefore the erasers that share its kernel.
Over all editors, \cref{prop:anisotropy} only brackets the minimax value, and a noise-weighted direction estimate may do better.

\end{document}